\documentclass{article} % For LaTeX2e

\usepackage[nonatbib,final]{nips_2016}
\usepackage[utf8]{inputenc} % allow utf-8 input
\usepackage[T1]{fontenc}    % use 8-bit T1 fonts
\usepackage{url}            % simple URL typesetting
\usepackage{booktabs}       % professional-quality tables
\usepackage{amsfonts}       % blackboard math symbols
\usepackage{nicefrac}       % compact symbols for 1/2, etc.
\usepackage{microtype}      % microtypography

\usepackage{url}

\usepackage{algorithm}
\usepackage{algorithmic}
\usepackage{amsmath,amssymb,amsthm}
\usepackage{psfrag}
\usepackage{color}
\usepackage{graphicx}
\usepackage{pdfpages}
\usepackage{multirow}
\usepackage{caption}

\newcommand{\A}{\ensuremath{\mathbf{A}}}

\newcommand{\C}{\ensuremath{\mathbf{C}}}

\newcommand{\I}{\ensuremath{\mathbf{I}}}

\newcommand{\M}{\ensuremath{\mathbf{M}}}
\newcommand{\N}{\ensuremath{\mathbf{N}}}

\newcommand{\Q}{\ensuremath{\mathbf{Q}}}

\newcommand{\T}{\ensuremath{\mathbf{T}}}
\newcommand{\U}{\ensuremath{\mathbf{U}}}

\newcommand{\X}{\ensuremath{\mathbf{X}}}
\newcommand{\Y}{\ensuremath{\mathbf{Y}}}

\renewcommand{\aa}{\ensuremath{\mathbf{a}}}
\renewcommand{\b}{\ensuremath{\mathbf{b}}}

\newcommand{\e}{\ensuremath{\mathbf{e}}}

\newcommand{\g}{\ensuremath{\mathbf{g}}}
\newcommand{\h}{\ensuremath{\mathbf{h}}}

\newcommand{\n}{\ensuremath{\mathbf{n}}}

\newcommand{\rr}{\ensuremath{\mathbf{r}}}
\renewcommand{\t}{\ensuremath{\mathbf{t}}}
\newcommand{\uu}{\ensuremath{\mathbf{u}}}
\newcommand{\vv}{\ensuremath{\mathbf{v}}}
\newcommand{\w}{\ensuremath{\mathbf{w}}}
\newcommand{\x}{\ensuremath{\mathbf{x}}}
\newcommand{\y}{\ensuremath{\mathbf{y}}}
\newcommand{\z}{\ensuremath{\mathbf{z}}}
\newcommand{\0}{\ensuremath{\mathbf{0}}}

\newcommand{\bphi}{\ensuremath{\boldsymbol{\phi}}}

\newcommand{\bpsi}{\ensuremath{\boldsymbol{\psi}}}

\newcommand{\bSigma}{\ensuremath{\boldsymbol{\Sigma}}}

\newcommand{\bbE}{\ensuremath{\mathbb{E}}}

\newcommand{\bbR}{\ensuremath{\mathbb{R}}}

\newcommand{\calO}{\ensuremath{\mathcal{O}}}

\newcommand{\abs}[1]{\left\lvert#1\right\rvert}
\newcommand{\norm}[1]{\left\lVert#1\right\rVert}

\newcommand{\caja}[4][1]{{%
    \renewcommand{\arraystretch}{#1}%
    \begin{tabular}[#2]{@{}#3@{}}%
      #4%
    \end{tabular}%
    }}

{%
\begin{list}{#1}{
\vspace{-\topsep}
\vspace{-\partopsep}
\setlength{\itemindent}{0cm}
\setlength{\rightmargin}{0cm}
\setlength{\listparindent}{0cm}
\settowidth{\labelwidth}{#1}
\setlength{\leftmargin}{\labelwidth}
\addtolength{\leftmargin}{\labelsep}
\setlength{\itemsep}{0cm}
}%
}%
{%
\end{list}
\vspace{-\topsep}
\vspace{-\partopsep}
}

{\begin{enumerate}%
}%
{\end{enumerate}}

\theoremstyle{plain}% default

\newtheorem*{prop*}{Proposition}

\theoremstyle{definition}

\newtheorem*{defn*}{Definition}

\newtheorem*{exmp*}{Example}

\newtheorem*{conj*}{Conjecture}

\theoremstyle{remark}

\newtheorem*{rmk*}{Remark}

\newcommand{\ie}{i.e.\@}
\newcommand{\eg}{e.g.\@}

\newcommand{\Sxy}{\bSigma_{xy}}
\newcommand{\Sxx}{\bSigma_{xx}}
\newcommand{\Syy}{\bSigma_{yy}}
\newcommand{\tu}{\tilde{\uu}}
\newcommand{\tv}{\tilde{\vv}}
\newcommand{\bu}{\bar{\uu}}
\newcommand{\bv}{\bar{\vv}}
\newcommand{\hu}{\hat{\uu}}
\newcommand{\hv}{\hat{\vv}}

\newcommand{\tphi}{\tilde{\bphi}}
\newcommand{\tpsi}{\tilde{\bpsi}}
\newcommand{\te}{\tilde{\epsilon}}

\newtheorem{theorem}{\textbf{Theorem}}
\newtheorem{lemma}[theorem]{\textbf{Lemma}}

\usepackage{mathtools}
\DeclarePairedDelimiter{\ceil}{\lceil}{\rceil}

\usepackage{eqparbox}

\newcommand{\AppGrad}{\texttt{AppGrad}}
\newcommand{\sAppGrad}{\texttt{s-AppGrad}}
\newcommand{\CCALin}{\texttt{CCALin}}
\newcommand{\ALSVR}{\texttt{ALS-VR}}
\newcommand{\ALSAVR}{\texttt{ALS-AVR}}
\newcommand{\ShiftVR}{\texttt{SI-VR}}
\newcommand{\ShiftAVR}{\texttt{SI-AVR}}

\title{Efficient Globally Convergent Stochastic Optimization for Canonical Correlation Analysis}

\author{  
  Weiran Wang$^{1*}$ \hspace{2em} Jialei Wang$^2$\thanks{The first two authors contributed equally.} \hspace{2em} Dan Garber$^1$ \hspace{2em} Nathan Srebro$^1$ \\
  $^1$Toyota Technological Institute at Chicago \hspace{4em} $^2$University of Chicago  \\  
  \texttt{\{weiranwang,dgarber,nati\}@ttic.edu} \hspace{3em} \texttt{jialei@uchicago.edu}
}

\begin{document}

\maketitle

\begin{abstract}
  We study the stochastic optimization of canonical correlation analysis (CCA), whose objective is nonconvex and does not decouple over training samples. 
  Although several stochastic gradient based optimization algorithms have been recently proposed to solve this problem, no global convergence guarantee was provided by any of them.
  Inspired by the alternating least squares/power iterations formulation of CCA, and the shift-and-invert preconditioning method for PCA, we propose two globally convergent meta-algorithms for CCA, both of which transform the original problem into sequences of least squares problems that need only be solved approximately. We instantiate the meta-algorithms with state-of-the-art SGD methods and obtain time complexities that significantly improve upon that of previous work. 
  Experimental results demonstrate their superior performance.
\end{abstract}

\section{Introduction}
\label{sec:intro}
\vspace*{-2ex}

Canonical correlation analysis (CCA,~\cite{Hotell36a}) and its extensions are ubiquitous techniques in scientific research areas for revealing the common sources of variability in multiple views of the same phenomenon. In CCA, the training set consists of paired observations from two views, denoted ${(\x_1,\y_1),\dots,(\x_N,\y_N)}$, where $N$ is the training set size, $\x_i\in \bbR^{d_x}$ and $\y_i\in \bbR^{d_y}$ for $i=1,\dots,N$. We also denote the data matrices for each view\footnote{We assume that $\X$ and $\Y$ are centered at the origin for notational simplicity; if they are not, we can center them as a pre-processing operation.} by $\X=[\x_1,\dots,\x_N] \in \bbR^{d_x \times N}$ and $\Y=[\y_1,\dots,\y_N] \in \bbR^{d_y \times N}$, and $d:=d_x + d_y$. 
The objective of CCA is to find linear projections of each view such that the correlation between the projections is maximized:
\begin{gather} \label{e:cca}
  \max_{\uu,\vv} \quad \uu^\top \Sxy \vv  \qquad\quad  \text{s.t.} \quad \uu^\top \Sxx \uu = \vv^\top \Syy \vv = 1 
\end{gather}
where $\Sxy= \frac{1}{N} \X\Y^\top$ is the cross-covariance matrix, $\Sxx= \frac{1}{N} \X\X^\top + \gamma_x \I$ and $\Syy= \frac{1}{N} \Y\Y^\top + \gamma_y \I$ are the auto-covariance matrices, and $(\gamma_x, \gamma_y)\ge 0$ are regularization parameters~\cite{Vinod76a}. 

We denote by $(\uu^*,\vv^*)$ the global optimum of~\eqref{e:cca}, which can be computed in closed-form. Define 
\begin{align}\label{e:def-T}
  \T:=\Sxx^{-\frac{1}{2}} \Sxy \Syy^{-\frac{1}{2}} \ \in \bbR^{d_x \times d_y},
\end{align}
and let $(\bphi,\bpsi)$ be the (unit-length) left and right singular vector pair associated with $\T$'s largest singular value $\rho_1$. Then the optimal objective value, \ie, the canonical correlation between the views, is $\rho_1$, achieved by $(\uu^*,\,\vv^*)=(\Sxx^{-\frac{1}{2}} \bphi,\,\Syy^{-\frac{1}{2}} \bpsi )$. Note that 
\begin{align*}
  \rho_1=\norm{\T}\le \norm{\Sxx^{-\frac{1}{2}} \X} \norm{\Syy^{-\frac{1}{2}} \Y}\le 1 .
\end{align*}
Furthermore, we are guaranteed to have $\rho_1 < 1$ if $(\gamma_x, \gamma_y) > 0$.

For large and high dimensional datasets, it is time and memory consuming to first explicitly form the matrix $\T$ (which requires eigen-decomposition of the covariance matrices) and then compute its singular value decomposition (SVD). For such datasets, it is desirable to develop stochastic algorithms that have efficient updates, converges fast, and takes advantage of the input sparsity. There have been recent attempts to solve~\eqref{e:cca} based on stochastic gradient descent (SGD) methods~\cite{Ma_15b,Wang_15c,Xie_15b}, but none of these work provides rigorous convergence analysis for their stochastic CCA algorithms.

The main contribution of this paper is the proposal of two globally convergent meta-algorithms for solving~\eqref{e:cca}, namely, alternating least squares (ALS, Algorithm~\ref{alg:meta-alsvr}) and shift-and-invert preconditioning (SI, Algorithm~\ref{alg:meta-shift-and-invert}), both of which transform the original problem~\eqref{e:cca} into sequences of least squares problems that need only be solved approximately. We instantiate the meta algorithms with state-of-the-art SGD methods and obtain efficient stochastic optimization algorithms for CCA.

In order to measure the alignments between an approximate solution $(\uu,\vv)$ and the optimum $(\uu^*,\vv^*)$, we assume that $\T$ has a positive singular value gap $\Delta := \rho_1 - \rho_2 \in (0,1]$ so its top left and right singular vector pair is unique (up to a change of sign). 
% Define the condition numbers:
% \begin{align*}
%   % \mu = \min\left( (\uu_0^\top \Sxx \uu^*)^2,\ (\vv_0^\top \Syy \vv^*)^2 \right), \quad 
%   \textstyle
%   \kappa & = \max\left( \frac{\max_i \norm{\x_i}^2 }{ \sigma_{\min} (\Sxx) },\ \frac{\max_i \norm{\y_i}^2 }{ \sigma_{\min} (\Syy) } \right), \qquad
%   \kappa^\prime = \max\left(\frac{\sigma_{\max} (\Sxx)}{\sigma_{\min} (\Sxx)},\ \frac{\sigma_{\max} (\Syy)}{\sigma_{\min} (\Syy)}\right), \\
%   % \tilde{\mu} = \frac{1}{4} \left( \uu_0^\top \Sxx \uu^* + \vv_0^\top \Syy \vv^* \right)^2, \quad 
%   \textstyle
%   \tilde{\kappa} & =   \frac{\max\limits_i\, \max \left(\norm{\x_i}^2, \norm{\y_i}^2 \right)}{\min \left( \sigma_{\min}(\Sxx), \sigma_{\min}(\Syy) \right) }, \qquad\qquad
%   \tilde{\kappa}^\prime = \frac{\max \left( \sigma_{\max}(\Sxx), \sigma_{\max}(\Syy) \right)} {\min \left( \sigma_{\min}(\Sxx), \sigma_{\min}(\Syy) \right) }.
% \end{align*}

Table~\ref{t:final} summarizes the time complexities of several algorithms for achieving $\eta$-suboptimal alignments, where $\tilde{\kappa} = \frac{\max\limits_i\, \max \left(\norm{\x_i}^2,\, \norm{\y_i}^2 \right)}{\min \left( \sigma_{\min}(\Sxx),\, \sigma_{\min}(\Syy) \right) }$ is the upper bound of condition numbers of least squares problems solved in all cases.\footnote{For the ALS meta-algorithm, its enough to consider a per-view conditioning. And when using AGD as the least squares solver, the time complexities dependends on $\sigma_{\max} (\Sxx)$ instead, which is less than $\max_i \norm{x_i}^2$.} 
We use the notation $\tilde{\calO}(\cdot)$ to hide poly-logarithmic dependencies (see Sec.~\ref{sec:svrg} and Sec.~\ref{sec:meta-shift-and-invert-time} for the hidden factors). 
Each time complexity may be preferrable in certain regime depending on the parameters of the problem.

\begin{table}[t]
  \caption{Time complexities of different algorithms for achieving $\eta$-suboptimal solution $(\uu,\vv)$ to CCA, \ie, $\min\left( (\uu^\top \Sxx \uu^*)^2,\ (\vv^\top \Syy \vv^*)^2 \right) \ge 1 - \eta$. GD=gradient descent, AGD=accelerated GD, SVRG=stochastic variance reduced gradient, ASVRG=accelerated SVRG. Note ASVRG provides speedup over SVRG only when $\tilde{\kappa} > N$, and we show the dominant term in its complexity.}
  \label{t:final}
  \centering
  \begin{tabular}{|c|c|l|} \hline
    Algorithm  & Least squares solver & \hspace*{4em} Time complexity \\ \hline \hline
    \AppGrad~\cite{Ma_15b} & GD & $\tilde{\calO} \left( d N \tilde{\kappa} \frac{\rho_1^2}{\rho_1^2 - \rho_2^2} \cdot \log \left(\frac{1}{\eta}\right) \right)  $ \ \ \ \ (local)\\ \hline
    \CCALin~\cite{Ge_16a} & AGD & $\tilde{\calO} \left( d N \sqrt{\tilde{\kappa}} \frac{\rho_1^2}{\rho_1^2 - \rho_2^2}  \cdot \log \left(\frac{1}{\eta}\right) \right)$ \\ \hline
    \multirow{3}{*}{\caja{c}{c}{This work:\\ Alternating least\\ squares (ALS)}}    & AGD     & $\tilde{\calO} \left( d N \sqrt{\tilde{\kappa}} \left( \frac{\rho_1^2}{\rho_1^2 - \rho_2^2}\right)^2 \cdot \log^2 \left(\frac{1}{\eta}\right) \right)$ \\
    & SVRG    & $\tilde{\calO} \left( d (N + \tilde{\kappa}) \left( \frac{\rho_1^2}{\rho_1^2 - \rho_2^2} \right)^2 \cdot \log^2 \left(\frac{1}{\eta}\right) \right)$ \\
    & ASVRG   & $\tilde{\calO} \left( d \sqrt{N \tilde{\kappa}} \left( \frac{\rho_1^2}{\rho_1^2 - \rho_2^2}\right)^2 \cdot \log^2 \left(\frac{1}{\eta}\right) \right)$ \\ \hline
    \multirow{3}{*}{\caja{c}{c}{This work:\\ Shift-and-invert \\ preconditioning (SI)}}   & AGD     & $\tilde{\calO} \left( d N \sqrt{\tilde{\kappa}} \sqrt{\frac{1}{\rho_1 - \rho_2}} \cdot \log^2 \left(\frac{1}{\eta}\right)  \right)$ \\
    & SVRG    &  $\tilde{\calO} \left( d \left( N + (\tilde{\kappa} \frac{1}{\rho_1 - \rho_2})^2 \right) \cdot \log^2 \left(\frac{1}{\eta}\right)  \right)$ \\
    & ASVRG  &  $\tilde{\calO} \left( d N^{\frac{3}{4}} \sqrt{\tilde{\kappa}} \sqrt{\frac{1}{\rho_1-\rho_2}} \cdot \log^2 \left(\frac{1}{\eta}\right) \right)$ \\
    \hline
  \end{tabular}
\vspace*{-0ex}
\end{table}

\textbf{Notations\ }  We use $\sigma_i (\A)$ to denote the $i$-th largest singular value of a matrix $\A$, and use $\sigma_{\max} (\A)$ and $\sigma_{\min} (\A)$ to denote the largest and smallest singular values of $\A$ respectively.

\section{Motivation: Alternating least squares}
\label{sec:motivation}
\vspace*{-2ex}

\begin{algorithm}[t]
  \caption{Alternating least squares for CCA.}
  \label{alg:als}
  \renewcommand{\algorithmicrequire}{\textbf{Input:}}
  \renewcommand{\algorithmicensure}{\textbf{Output:}}
  \begin{algorithmic}
    \REQUIRE Data matrices $\X \in \bbR^{d_x \times N}$, $\Y \in \bbR^{d_y \times N}$, regularization parameters $(\gamma_x, \gamma_y)$.
    \STATE Initialize $\tu_0 \in \bbR^{d_x},\quad \tv_0 \in \bbR^{d_y}$. \hfill \COMMENT{\tphi_0,\ \tpsi_0}
    \STATE $\uu_0 \leftarrow \tu_0 / \sqrt{\tu_0^\top \Sxx \tu_0},\quad \vv_0 \leftarrow \tv_0 / \sqrt{\tv_0^\top \Syy \tv_0}$ \hfill \COMMENT{\bphi_0 \leftarrow \tphi_0 / \norm{\tphi_0},\ \bpsi_0 \leftarrow \tpsi_0 / \norm{\tpsi_0}}
    \FOR{$t=1,2,\dots,T$}
    \STATE $\tu_{t} \leftarrow \Sxx^{-1} \Sxy \vv_{t-1}$ \hfill \COMMENT{\tphi_{t} \leftarrow  \Sxx^{-\frac{1}{2}} \Sxy \Syy^{-\frac{1}{2}} \bpsi_{t-1}}
    \STATE $\tv_{t} \leftarrow \Syy^{-1} \Sxy^\top \uu_{t-1}$  \hfill \COMMENT{\tpsi_{t} \leftarrow  \Syy^{-\frac{1}{2}} \Sxy^\top \Sxx^{-\frac{1}{2}} \bphi_{t-1}}
    \STATE $\uu_{t} \leftarrow \tu_{t} / \sqrt{\tu_{t}^\top \Sxx \tu_{t}},\quad \vv_{t} \leftarrow \tv_{t} / \sqrt{\tv_{t}^\top \Syy \tv_{t}}$ \hfill \COMMENT{\bphi_{t} \leftarrow \tphi_{t} / \norm{\tphi_{t}},\ \bpsi_{t} \leftarrow \tpsi_{t} / \norm{\tpsi_{t}}}
    \ENDFOR  
    \ENSURE $(\uu_{T},\vv_{T}) \rightarrow (\uu^*,\vv^*)$ as $T \rightarrow \infty$. \hfill \COMMENT{\texttt{$(\bphi_{T},\bpsi_{T}) \rightarrow (\bphi,\bpsi)$}}
  \end{algorithmic}
\end{algorithm}

Our solution to~\eqref{e:cca} is inspired by the alternating least squares (ALS) formulation of CCA~\cite[Algorithm 5.2]{GolubZha95a}, as shown in Algorithm~\ref{alg:als}. 
Let the nonzero singular values of $\T$ be $1 \ge \rho_1 \ge \rho_2 \ge \dots \ge \rho_r >0$, where $r = \text{rank}(\T) \le \min(d_x,d_y)$, and the corresponding (unit-length) left and right singular vector pairs be $(\aa_1,\b_1),\dots,(\aa_r,\b_r)$, with $\aa_1$=$\bphi$ and $\b_1=\bpsi$. 
Define 
\begin{align} \label{e:C-def}
  \C = \left[
    \begin{array}{cc}
      \0 & \T \\
      \T^\top & \0 
    \end{array}
  \right] \in \bbR^{ d \times d }.
\end{align}
It is straightforward to check that the nonzero eigenvalues of $\C$ are:
\begin{align*}
  \rho_1 \ge \dots \ge \rho_r \ge -\rho_r \ge \dots \ge -\rho_1,
\end{align*}
with corresponding eigenvectors 
$
  \frac{1}{\sqrt{2}}
  \left[ \begin{array}{c} \aa_1\\ \b_1 \end{array} \right], \ \dots,\ 
  \frac{1}{\sqrt{2}} 
  \left[ \begin{array}{c} \aa_r\\ \b_r \end{array} \right], \ 
  \frac{1}{\sqrt{2}} 
  \left[ \begin{array}{c} \aa_r\\ - \b_r \end{array} \right], \ \dots,\  
  \frac{1}{\sqrt{2}} 
  \left[ \begin{array}{c} \aa_1\\ - \b_1 \end{array} \right].
$

The key observation is that Algorithm~\ref{alg:als} effectively runs a variant of power iterations on $\C$ to extract its top eigenvector. To see this, make the following change of variables 
\begin{align} \label{e:change-variable}
\bphi_{t}=\Sxx^{\frac{1}{2}} \uu_{t}, \qquad \bpsi_{t}=\Syy^{\frac{1}{2}} \vv_{t}, \qquad 
\tphi_{t}=\Sxx^{\frac{1}{2}} \tu_{t}, \qquad \tpsi_{t}=\Syy^{\frac{1}{2}} \tv_{t}.
\end{align}
Then we can equivalently rewrite the steps of Algorithm~\ref{alg:als} in the new variables as in $\{\}$ of each line.

Observe that the iterates are updated as follows from step $t-1$ to step $t$:
\begin{align}
  \left[ \begin{array}{c} \tphi_{t} \\ \tpsi_{t} \end{array} \right] \leftarrow
  \left[
    \begin{array}{cc}
      \0 & \T \\
      \T^\top & \0 
    \end{array}
  \right] \left[ \begin{array}{c} \bphi_{t-1} \\ \bpsi_{t-1} \end{array} \right], \qquad
  \left[ \begin{array}{c} \bphi_{t} \\ \bpsi_{t} \end{array} \right] \leftarrow
  \left[ \begin{array}{c} \tphi_{t}/ ||\tphi_{t}|| \\ \tpsi_{t}/ ||\tpsi_{t}|| \end{array} \right].
\end{align}
Except for the special normalization steps which rescale the two sets of variables separately, Algorithm~\ref{alg:als} is very similar to the power iterations~\cite{GolubLoan96a}.

We show the convergence rate of ALS below (see its proof in Appendix~\ref{append:proof-alternating-least-squares-exact}). The first measure of progress is the alignment of $\bphi_{t}$ to $\bphi$ and the alignment of $\bpsi_{t}$ to $\bpsi$, \ie, $(\bphi_{t}^\top \bphi)^2=(\uu_{t}^\top \Sxx \uu^*)^2$ and $(\bpsi_{t}^\top \bpsi)^2=(\vv_{t}^\top \Syy \vv^*)^2$. The maximum value for such alignments is $1$, achieved when the iterates completely align with the optimal solution. 
The second natural measure of progress is the objective of \eqref{e:cca}, \ie, $\uu_t^\top \Sxy \vv_t$, with the maximum value being $\rho_1$. 

\begin{theorem}[Convergence of Algorithm~\ref{alg:als}] \label{thm:alternating-least-squares-exact}
  Let $\mu:=\min\left( (\uu_0^\top \Sxx \uu^*)^2,\ (\vv_0^\top \Syy \vv^*)^2 \right) > 0$.\footnote{One can show that $\mu$ is bounded away from $0$ with high probability using random initialization $(\uu_0,\,\vv_0)$.} Then for $t \ge \ceil{ \frac{\rho_1^2}{\rho_1^2-\rho_2^2} \log\left( \frac{1}{\mu \eta} \right) }$, we have in Algorithm~\ref{alg:als} that $\min\left( (\uu_{t}^\top \Sxx \uu^*)^2,\ (\vv_{t}^\top \Syy \vv^*)^2 \right) \ge 1 - \eta$, and $\uu_t^\top \Sxy \vv_t \ge \rho_1 (1 - 2 \eta)$.
\end{theorem}

\paragraph{Remarks} We have assumed a nonzero singular value gap in Theorem~\ref{thm:alternating-least-squares-exact} to obtain linear convergence in both the alignments and the objective. When there exists no singular value gap, the top singular vector pair is not unique and it is no longer meaningful to measure the alignments. Nonetheless, it is possible to extend our proof to obtain sublinear convergence for the objective in this case.

Observe that, besides the steps of normalization to unit length, the basic operation in each iteration of Algorithm~\ref{alg:als} is of the form $\tu_{t} \leftarrow \Sxx^{-1} \Sxy \vv_{t-1} =(\frac{1}{N} \X \X^\top + \gamma_x \I)^{-1} \frac{1}{N} \X \Y^\top \vv_{t-1}$, which is equivalent to solving the following regularized least squares (ridge regression) problem
\begin{align}\label{e:lsq}
  \min_{\uu}\ \frac{1}{2N} \norm{\uu^\top \X - \vv_{t-1}^\top \Y}^2 + \frac{\gamma_x}{2} \norm{\uu}^2 \equiv \min_{\uu}\  \frac{1}{N} \sum_{i=1}^N \frac{1}{2} \abs{\uu^\top \x_i - \vv_{t-1}^\top \y_i}^2 + \frac{\gamma_x}{2} \norm{\uu}^2.
\end{align}
In the next section, we show that, to maintain the convergence of ALS, it is unnecessary to solve the least squares problems exactly. This enables us to use state-of-the-art SGD methods for solving~\eqref{e:lsq} to sufficient accuracy, and to obtain a globally convergent stochastic algorithm for CCA.

\section{Our algorithms}
\label{sec:algorithms}
\vspace*{-1ex}

\subsection{Algorithm I: Alternating least squares (ALS) with variance reduction}
\label{sec:alg-alsvr}
\vspace*{-1ex}

\begin{algorithm}[t]
  \caption{The alternating least squares (ALS) meta-algorithm for CCA.}
  \label{alg:meta-alsvr}
  \renewcommand{\algorithmicrequire}{\textbf{Input:}}
  \renewcommand{\algorithmicensure}{\textbf{Output:}}
  \begin{algorithmic}
    \REQUIRE Data matrices $\X \in \bbR^{d_x \times N}$, $\Y \in \bbR^{d_y \times N}$, regularization parameters $(\gamma_x, \gamma_y)$.
    \STATE Initialize $\tu_0 \in \bbR^{d_x},\ \tv_0 \in \bbR^{d_y}$.
    \STATE $\tu_0 \leftarrow \tu_0 / \sqrt{\tu_0^\top \Sxx \tu_0}, \qquad \tv_0 \leftarrow \tv_0 / \sqrt{\tv_0^\top \Syy \tv_0},\qquad \uu_0 \leftarrow \tu_0, \qquad  \vv_0 \leftarrow \tv_0$
    \FOR{$t=1,2,\dots,T$}
    \STATE Solve $\displaystyle \min_{\uu} \ f_{t}(\uu):=\frac{1}{2N} \norm{\uu^\top \X - \vv_{t-1}^\top \Y}^2 + \frac{\gamma_x}{2} \norm{\uu}^2$ with initialization $\tu_{t-1}$, and output approximate solution $\tu_{t}$ satisfying $\ f_{t} (\tu_{t}) \le \min_{\uu} \ f_{t} (\uu) + \epsilon$.
    \STATE Solve $\displaystyle \min_{\vv} \ g_{t}(\vv):=\frac{1}{2N} \norm{\vv^\top \Y - \uu_{t-1}^\top \X}^2 + \frac{\gamma_y}{2} \norm{\vv}^2$ with initialization $\tv_{t-1}$, and output approximate solution $\tv_{t}$ satisfying $g_{t} (\tv_{t}) \le \min_{\vv} \ g_{t} (\vv) + \epsilon$.
    \STATE $\uu_{t} \leftarrow \tu_{t} / \sqrt{\tu_{t}^\top \Sxx \tu_{t}}, \qquad \vv_{t} \leftarrow \tv_{t} / \sqrt{\tv_{t}^\top \Syy \tv_{t}}$
    \ENDFOR
    \ENSURE $(\uu_{T}, \vv_{T})$ is the approximate solution to CCA.
  \end{algorithmic}
\end{algorithm}

Our first algorithm consists of two nested loops. The outer loop runs inexact power iterations while the inner loop uses advanced stochastic optimization methods, \eg, stochastic variance reduced gradient (SVRG,~\cite{JohnsonZhang13a}) to obtain approximate matrix-vector multiplications. A sketch of our algorithm is provided in Algorithm~\ref{alg:meta-alsvr}. We make the following observations from this algorithm.

% \begin{itemize}
\textbf{Connection to previous work\ } At step $t$, if we optimize $f_{t}(\uu)$ and $g_{t}(\vv)$ crudely by a single batch gradient descent step from the initialization $(\tu_{t-1},\tv_{t-1})$, we obtain the following update rule:
\begin{align*}
  \tu_{t} \leftarrow \tu_{t-1} - 2 \xi\, \X (\X^\top \tu_{t-1} - \Y^\top \vv_{t-1})/N,\qquad \uu_{t} \leftarrow \tu_{t} / \sqrt{\tu_{t}^\top \Sxx \tu_{t}} \\
  \tv_{t} \leftarrow \tv_{t-1} - 2 \xi\, \Y (\Y^\top \tv_{t-1} - \X^\top \uu_{t-1})/N,\qquad \vv_{t} \leftarrow \tv_{t} / \sqrt{\tv_{t}^\top \Syy \tv_{t}}
\end{align*} 
where $\xi > 0$ is the stepsize (assuming $\gamma_x=\gamma_y=0$). 
This coincides with the \AppGrad\ algorithm of~\cite[Algorithm~3]{Ma_15b}, for which only local convergence is shown. Since the objectives $f_{t}(\uu)$ and $g_{t}(\vv)$ decouple over training samples, it is convenient to apply SGD methods to them. This observation motivated the stochastic CCA algorithms of~\cite{Ma_15b,Wang_15c}. We note however, no global convergence guarantee was shown for these stochastic CCA algorithms, and the key to our convergent algorithm is to solve the least squares problems to \emph{sufficient} accuracy.

\textbf{Warm-start\ } Observe that for different $t$, the least squares problems $f_t(\uu)$ only differ in their targets as $\vv_t$ changes over time. Since $\vv_{t-1}$ is close to $\vv_{t}$ (especially when near convergence), we may use $\tu_{t}$ as initialization for minimizing $f_{t+1}(\uu)$ with an iterative algorithm.

\textbf{Normalization\ } At the end of each outer loop, Algorithm~\ref{alg:meta-alsvr} implements exact normalization of the form $\uu_{t} \leftarrow \tu_{t} / \sqrt{\tu_{t}^\top \Sxx \tu_{t}}$ to ensure the constraints, where $\tu_{t}^\top \Sxx \tu_{t}=\frac{1}{N} (\tu_{t}^\top \X) (\tu_{t}^\top \X)^\top + \gamma_x \norm{\tu_{t}}^2$ requires computing the projection of the training set $\tu_{t}^\top \X$. However, this does not introduce extra computation because we also compute this projection for the batch gradient used by SVRG (at the beginning of time step $t+1$). In contrast, the stochastic algorithms of~\cite{Ma_15b, Wang_15c} (possibly adaptively) estimate the covariance matrix from a minibatch of training samples and use the estimated covariance for normalization. This is because their algorithms perform normalizations after each update and thus need to avoid computing the projection of the entire training set frequently. But as a result, their inexact normalization steps introduce noise to the algorithms.

\textbf{Input sparsity\ } For high dimensional sparse data (such as those used in natural language processing~\cite{LuFoster14a}), an advantage of gradient based methods over the closed-form solution is that the former takes into account the input sparsity. For sparse inputs, the time complexity of our algorithm depends on $nnz(\X,\Y)$, \ie, the total number of nonzeros in the inputs instead of $dN$.

\textbf{Canonical ridge\ } When $(\gamma_x, \gamma_y)> 0$, $f_{t} (\uu)$ and $g_{t} (\vv)$ are guaranteed to be strongly convex due to the $\ell_2$ regularizations, in which case SVRG converges linearly. It is therefore beneficial to use small nonzero regularization for improved computational efficiency, especially for high dimensional datasets where inputs $\X$ and $\Y$ are approximately low-rank.

% \subsection{Analysis of inexact power iterations}
\textbf{Convergence\ } By the analysis of inexact power iterations where the least squares problems are solved (or the matrix-vector multiplications are computed) only up to necessary accuracy, we provide the following theorem for the convergence of Algorithm~\ref{alg:meta-alsvr} (see its proof in Appendix~\ref{append:proof-alternating-least-squares-inexact}). The key to our analysis is to bound the distances between the iterates of Algorithm~\ref{alg:meta-alsvr} and that of Algorithm~\ref{alg:als} at all time steps,  and when the errors of the least squares problems are sufficiently small (at the level of $\eta^2$), the iterates of the two algorithms have the same quality.

\begin{theorem}[Convergence of Algorithm~\ref{alg:meta-alsvr}] \label{thm:alternating-least-squares-inexact}
Fix $\smash{ T \ge \ceil{\frac{\rho_1^2}{\rho_1^2 - \rho_2^2} \log\left( \frac{2}{\mu \eta} \right) } }$, and set $\epsilon (T) \le  \frac{\eta^2 \rho_r^2}{128} \cdot \left( \frac{(2\rho_1/\rho_r) - 1}{ (2\rho_1/\rho_r)^T - 1} \right)^2$ in Algorithm~\ref{alg:meta-alsvr}. 
Then we have $\uu_{T}^\top \Sxx \uu_{T}=\vv_{T}^\top \Syy \vv_{T}=1$, $\min\left( (\uu_T^\top \Sxx \uu^*)^2,\ (\vv_T^\top \Syy \vv^*)^2 \right) \ge 1 - \eta$, and $\uu_T^\top \Sxy \vv_T \ge \rho_1 (1 - 2 \eta)$.
\end{theorem}

\subsubsection{Stochastic optimization of regularized least squares}
\label{sec:svrg}
\vspace*{-1ex}

We now discuss the inner loop of Algorithm~\ref{alg:meta-alsvr}, which approximately solves problems of the form~\eqref{e:lsq}. Owing to the finite-sum structure of~\eqref{e:lsq}, several stochastic optimization methods such as SAG~\cite{Schmid_13a}, SDCA~\cite{ShalevZhang13b} and SVRG~\cite{JohnsonZhang13a}, provide linear convergence rates. All these algorithms can be readily applied to~\eqref{e:lsq}; we choose SVRG since it is memory efficient and easy to implement. We also apply the recently developed accelerations techniques for first order optimization methods~\cite{Frostig_15a,Lin_15a} to obtain an accelerated SVRG (ASVRG) algorithm. 
We give the sketch of SVRG for~\eqref{e:lsq} in Appendix~\ref{append:alg-svrg}. % (Algorithm~\ref{alg:svrg}). 

Note that $f(\uu)=\frac{1}{N} \sum_{i=1}^N f^i(\uu)$ where each component $f^i(\uu)=\frac{1}{2} \abs{\uu^\top \x_i - \vv^\top \y_i }^2 + \frac{\gamma_x}{2} \norm{\uu}^2 $ is $\norm{\x_i}^2$-smooth, and $f(\uu)$ is $\sigma_{\min} (\Sxx)$-strongly convex\footnote{We omit the regularization in these constants, which are typically very small, to have concise expressions.} with $\sigma_{\min} (\Sxx) \ge \gamma_x$. % is the minimum eigenvalue of $\Sxx$. 
We show in Appendix~\ref{append:meta-alternating-least-squares-initial-suboptimality} that the initial suboptimality for minimizing $f_t (\uu)$ is upper-bounded by constant when using the warm-starts. We quote the convergence rates of SVRG~\cite{JohnsonZhang13a} and ASVRG~\cite{Lin_15a} below.
\begin{lemma}\label{lem:svrg}
  The SVRG algorithm~\cite{JohnsonZhang13a} % detailed in Algorithm~\ref{alg:svrg} 
  finds a vector $\tu$ satisfying\footnote{The expectation is taken over random sampling of component functions. High probability error bounds can be obtained using the Markov's inequality.} $\bbE[f(\tu)]-\min_{\uu} f(\uu) \le \epsilon$ in time 
  $\calO\left(
      d_x \left( N + \kappa_x \right) \log \left(\frac{1}{\epsilon}\right)
    \right)$ 
  where $\kappa_x = \frac{\max_i \norm{\x_i}^2 }{ \sigma_{\min} (\Sxx) }$. 
  The ASVRG algorithm~\cite{Lin_15a} finds a such solution in time 
    $\calO\left(
      d_x \sqrt{ N  \kappa_x} \log \left(\frac{1}{\epsilon}\right)
    \right)$.
\end{lemma}

\paragraph{Remarks} As mentioned in~\cite{Lin_15a}, the acceleration version provides speedup over normal SVRG only when $\kappa_x > N$ and we only show the dominant term in the above complexity.

By combining the iteration complexity of the outer loop (Theorem~\ref{thm:alternating-least-squares-inexact}) and the time complexity of the inner loop (Lemma~\ref{lem:svrg}), we obtain the total time complexity of $\smash{ \tilde{\calO} \left( d \left(N + \kappa \right) \left( \frac{\rho_1^2}{\rho_1^2 - \rho_2^2} \right)^2 \cdot \log^2 \left(\frac{1}{\eta}\right) \right) }$ for ALS+SVRG 
and ${ \tilde{\calO} \left( d \sqrt{N \kappa} \left( \frac{\rho_1^2}{\rho_1^2 - \rho_2^2} \right)^2 \cdot \log^2 \left(\frac{1}{\eta}\right) \right) }$ for ALS+ASVRG, where ${ \kappa := \max\left( \frac{\max_i \norm{\x_i}^2 }{ \sigma_{\min} (\Sxx) },\, \frac{\max_i \norm{\y_i}^2 }{ \sigma_{\min} (\Syy) } \right) }$ and $\tilde{\calO}(\cdot)$ hides poly-logarithmic dependences on $\frac{1}{\mu}$ and $\frac{1}{\rho_r}$. % $\mathtt{polylog}\left(\frac{1}{\mu}, \frac{1}{\rho_r}\right)$
Our algorithm does not require the initialization to be close to the optimum and converges globally. For comparison, the locally convergent \AppGrad\ has a time complexity \cite[Theorem~2.1]{Ma_15b} of $\tilde{\calO} \left( d N \kappa^\prime \frac{\rho_1^2}{\rho_1^2 - \rho_2^2}  \cdot \log \left(\frac{1}{\eta}\right) \right)$, where $\kappa^\prime := \max\left(\frac{\sigma_{\max} (\Sxx)}{\sigma_{\min} (\Sxx)},\ \frac{\sigma_{\max} (\Syy)}{\sigma_{\min} (\Syy)}\right)$. Note, in this complexity, the dataset size $N$ and the least squares condition number $\kappa^\prime$ are multiplied together because \AppGrad\ essentially uses batch gradient descent as the least squares solver. Within our framework, we can use accelerated gradient descent (AGD,~\cite{Nester04a}) instead and obtain a globally convergent algorithm with a total time complexity of $\tilde{\calO} \left( d N \sqrt{\kappa^\prime} \left( \frac{\rho_1^2}{\rho_1^2 - \rho_2^2} \right)^2 \cdot \log^2 \left(\frac{1}{\eta}\right) \right)$.

\subsection{Algorithm II: Shift-and-invert preconditioning (SI) with variance reduction}
\label{sec:alg-shift-and-invert}
\vspace*{-1ex}

The second algorithm is inspired by the shift-and-invert preconditioning method for PCA~\cite{GarberHazan15c,Jin_15a}. Instead of running power iterations on $\C$ as defined in~\eqref{e:C-def}, we will be running power iterations on 
\begin{align}
  \label{e:M-def}
  \M_{\lambda} =
  \left(
    \lambda \I - \C
  \right)^{-1}
  =
  \left[ 
    \begin{array}{cc}
      \lambda \I & -\T \\
      -\T^\top & \lambda \I
    \end{array}
  \right]^{-1}
  \in \bbR^{ d \times d },
\end{align}
where $\lambda>\rho_1$. It is straightforward to check that $\M_{\lambda}$ is positive definite and its eigenvalues are:
\begin{align*}
  \frac{1}{\lambda-\rho_1} \ge \dots \ge \frac{1}{\lambda-\rho_r} \ge \dots \ge \frac{1}{\lambda+\rho_r} \ge \dots \ge \frac{1}{\lambda+\rho_1},
\end{align*}
with % corresponding 
eigenvectors 
$
  \frac{1}{\sqrt{2}}
  \left[ \begin{array}{c} \aa_1\\ \b_1 \end{array} \right], \ \dots,\ 
  \frac{1}{\sqrt{2}} 
  \left[ \begin{array}{c} \aa_r\\ \b_r \end{array} \right], \ \dots,\ 
  \frac{1}{\sqrt{2}} 
  \left[ \begin{array}{c} \aa_r\\ - \b_r \end{array} \right], \ \dots,\  
  \frac{1}{\sqrt{2}} 
  \left[ \begin{array}{c} \aa_1\\ - \b_1 \end{array} \right].
$

The main idea behind shift-and-invert power iterations is that when $\lambda - \rho_1 = c (\rho_1 - \rho_2)$ with $c \sim \calO(1)$, the relative eigenvalue gap of $\M_{\lambda}$ is large and so power iterations on $\M_{\lambda}$ converges quickly. % (as opposed to power iterations on $\C$). 
Our shift-and-invert preconditioning (SI) meta-algorithm for CCA is sketched in Algorithm~\ref{alg:meta-shift-and-invert} (in Appendix~\ref{append:alg-shift-and-invert} due to space limit) and it proceeds in two phases. 

\subsubsection{Phase I: shift-and-invert preconditioning for eigenvectors of $\M_{\lambda}$} 
\vspace*{-1ex}

%% In the first phase, we extract the top eigenvector of $\M_{\lambda}$, ignoring its block structure. 
Using an estimate of the singular value gap $\tilde{\Delta}$ and starting from an over-estimate of $\rho_1$ ($1+\tilde{\Delta}$ suffices), the algorithm gradually shrinks $\lambda_{(s)}$ towards $\rho_1$ by crudely estimating the leading eigenvector/eigenvalues of each $\M_{\lambda_{(s)}}$ along the way and shrinking the gap $\lambda_{(s)} - \rho_1$, until we reach a $\lambda_{(f)} \in (\rho_1, \rho_1 + c (\rho_1 - \rho_2))$ where $c \sim \calO(1)$. 
Afterwards, the algorithm fixes $\lambda_{(f)}$ and runs inexact power iterations on $\M_{\lambda_{(f)}}$ to obtain an accurate estimate of its leading eigenvector. 
Note in this phase, power iterations implicitly operate on the concatenated variables ${ \frac{1}{\sqrt{2}} \left[ \begin{array}{c} \Sxx^{\frac{1}{2}} \tu_{t} \\ \Syy^{\frac{1}{2}} \tv_{t} \end{array} \right]}$ and $\smash{ \frac{1}{\sqrt{2}} \left[ \begin{array}{c} \Sxx^{\frac{1}{2}} \uu_{t} \\ \Syy^{\frac{1}{2}} \vv_{t} \end{array} \right] }$ in $\bbR^{d}$ (but without ever computing $\Sxx^{\frac{1}{2}}$ and $\Syy^{\frac{1}{2}}$).

% \begin{align*}
%   \left[ \begin{array}{c} \tphi_{t} \\ \tpsi_{t} \end{array} \right] \leftarrow
%   \M_{\lambda} \left[ \begin{array}{c} \bphi_{t-1} \\ \bpsi_{t-1} \end{array} \right], 
% \end{align*}

\paragraph{Matrix-vector multiplication\ } The matrix-vector multiplications in Phase I have the form
\begin{align} \label{e:new-lsq}
  \left[
    \begin{array}{c}\tu_t \\ \tv_t \end{array}
  \right]
  \leftarrow 
  \left[
    \begin{array}{cc}
      \lambda \Sxx & - \Sxy \\
      - \Sxy^\top & \lambda \Syy
    \end{array}
  \right]^{-1}
  \left[
    \begin{array}{cc}
      \Sxx &  \\
      & \Syy
    \end{array}
  \right]
  \left[
    \begin{array}{c}\uu_{t-1} \\ \vv_{t-1} \end{array}
  \right],
\end{align}
where $\lambda$ varies over time in order to locate $\lambda_{(f)}$. This is equivalent to solving 
\begin{align*} % \label{e:new-lsq-2}
  \left[
    \begin{array}{c}\tu_t \\ \tv_t \end{array}
  \right]
  \leftarrow 
  \min_{\uu,\vv}\; \frac{1}{2}
  \left[\uu^\top \vv^\top \right]
  \left[
    \begin{array}{cc}
      \lambda \Sxx & - \Sxy \\
      - \Sxy^\top & \lambda \Syy
    \end{array}
  \right]
  \left[
    \begin{array}{c}
      \uu \\ \vv
    \end{array}
  \right]
  - \uu^\top \Sxx \uu_{t-1} - \vv^\top \Syy \vv_{t-1}.
\end{align*}

And as in ALS, this least squares problem can be further written as finite-sum:
\begin{gather}\label{e:new-lsq-3}
  \min_{\uu,\vv} \quad h_t(\uu,\vv) = \frac{1}{N} \sum_{i=1}^N h_t^i(\uu,\vv) \qquad \text{where} \\
  h_t^i(\uu,\vv) =
  \frac{1}{2}
  \left[\uu^\top \vv^\top \right]
  \left[
    \begin{array}{cc}
      \lambda \left( \x_i \x_i^\top + \gamma_x \I \right) & - \x_i \y_i^\top \\
      - \y_i \x_i^\top & \lambda \left( \y_i \y_i^\top + \gamma_y \I \right)
    \end{array}
  \right]
  \left[
    \begin{array}{c}
      \uu \\ \vv
    \end{array}
  \right] 
  - \uu^\top \Sxx \uu_{t-1} - \vv^\top \Syy \vv_{t-1}.
  \nonumber
\end{gather}
We could directly apply SGD methods to this problem as before.

\textbf{Normalization\ } The normalization steps in Phase I have the form
% \begin{align*}
%   % \label{e:normal-1}
%   \left[ \begin{array}{c} \bphi_{t} \\ \bpsi_{t} \end{array} \right] \leftarrow
%   \sqrt{2} \left[ \begin{array}{c} \tphi_{t} \\ \tpsi_{t} \end{array} \right] \bigg/ 
%   \norm{\left[ \begin{array}{c} \tphi_{t} \\ \tpsi_{t} \end{array} \right]},
% \end{align*}
\begin{align*}
  % \label{e:normal-2}
  \left[ \begin{array}{c} \uu_{t} \\ \vv_{t} \end{array} \right] \leftarrow
  \sqrt{2} \left[ \begin{array}{c} \tu_{t} \\ \tv_{t} \end{array} \right] \bigg/ 
  \sqrt{ \tu_{t}^\top \Sxx \tu_{t} + \tv_{t}^\top \Syy \tv_{t} },
\end{align*}
and so the following remains true for the normalized iterates in Phase I:
\begin{align} \label{e:normal-3}
  \uu_{t}^\top \Sxx \uu_{t} + \vv_{t}^\top \Syy \vv_{t} = 2, \qquad  \text{for}\quad t=1,\dots,T.
\end{align}
Unlike the normalizations in ALS, the iterates $\uu_t$ and $\vv_t$ in Phase I do \emph{not} satisfy the original CCA constraints, and this is taken care of in Phase II. 

We have the following convergence guarantee for Phase I (see its proof in Appendix~\ref{append:proof-shift-and-invert-phase-I}).

\begin{theorem}[Convergence of Algorithm~\ref{alg:meta-shift-and-invert}, Phase I] \label{thm:shift-and-invert-phase-I}
  Let $\Delta=\rho_1-\rho_2 \in (0,\, 1]$, and $\tilde{\mu} := \frac{1}{4} \left( \uu_0^\top \Sxx \uu^* + \vv_0^\top \Syy \vv^* \right)^2 > 0$, and $\tilde{\Delta} \in [c_1 \Delta,\, c_2 \Delta]$ where $0< c_1 \le c_2 \le 1$.
  Set $m_1 = \ceil{ 8 \log \left( \frac{16}{\tilde{\mu}} \right) }$, $m_2 = \ceil{ \frac{5}{4}  \log \left( \frac{128}{\tilde{\mu} \eta^2} \right) }$, and 
  % \begin{align*}
$
    \te \le \min \left(  
      \frac{1}{3084} \left( \frac{\tilde{\Delta}}{18} \right)^{m_1-1},\ 
      \frac{ \eta^4 }{4^{10}} \left( \frac{\tilde{\Delta}}{18} \right)^{m_2-1}
       \right)
$
  % \end{align*}
in Algorithm~\ref{alg:meta-shift-and-invert}.
  Then the $(\uu_{T}, \vv_{T})$ output by Phase I of Algorithm~\ref{alg:meta-shift-and-invert} satisfies \eqref{e:normal-3} and 
  \begin{align} \label{e:phaseI-2-main}
    \frac{1}{4} (\uu_{T}^\top \Sxx \uu^* + \vv_{T}^\top \Syy \vv^*)^2 \ge 1 - \frac{\eta^2}{64},
  \end{align}
  and the number of calls to the least squares solver of % the least squares problems 
  $h_t(\uu,\vv)$ is 
  % \begin{align*}
  $\calO \left( \log\left(\frac{1}{\tilde{\mu}}\right) \log \left( \frac{1}{\Delta} \right) + \log\left( \frac{1}{\tilde{\mu} \eta^2} \right) \right)$.
  % \end{align*}

%% Moreover, all $\lambda_{(s)}$ used in Algorithm~\ref{alg:meta-shift-and-invert} satisfy that  the condition number of $\M_{\lambda_{(s)}}$ is $\frac{\lambda_{(s)} + \rho_1}{\lambda_{(s)} - \rho_1} \le \frac{9}{\tilde{\Delta}}$.
\end{theorem}

\subsubsection{Phase II: final normalization}
\vspace*{-1ex}

In order to satisfy the CCA constraints, we perform a last normalization
\begin{align}\label{e:final-normal}
  \hu \leftarrow \uu_{T} / \sqrt{\uu_{T}^\top \Sxx \uu_{T}}, \qquad
  \hv \leftarrow \vv_{T} / \sqrt{\vv_{T}^\top \Syy \vv_{T}}.
\end{align}
And we output $(\hu,\hv)$ as our final approximate solution to~\eqref{e:cca}. 
We show that this step does not cause much loss in the alignments, as stated below (see it proof in Appendix~\ref{append:proof-shift-and-invert-phase-II}).

\begin{theorem}[Convergence of Algorithm~\ref{alg:meta-shift-and-invert}, Phase II] \label{thm:shift-and-invert-phase-II}
  Let Phase I of Algorithm~\ref{alg:meta-shift-and-invert} outputs $(\uu_T, \vv_T)$ that satisfy~\eqref{e:phaseI-2-main}. Then after~\eqref{e:final-normal},
we obtain an approximate solution $(\hu,\hv)$ to~\eqref{e:cca} such that $\hu^\top \Sxx \hu = \hv^\top \Syy \hv = 1$,  $\min\left( (\hu^\top \Sxx \uu^*)^2,\, (\hv^\top \Syy \vv^*)^2 \right) \ge 1 - \eta$, and $\hu^\top \Sxy \hv \ge \rho_1 (1 - 2 \eta)$.
\end{theorem}

\subsubsection{Time complexity}
\label{sec:meta-shift-and-invert-time}
\vspace*{-1ex}

We have shown in Theorem~\ref{thm:shift-and-invert-phase-I} that Phase I only approximately solves a small number of instances of~\eqref{e:new-lsq-3}. The normalization steps~\eqref{e:normal-3} require computing the projections of the traning set which are reused for computing batch gradients of~\eqref{e:new-lsq-3}. The final normalization~\eqref{e:final-normal} is done only once and costs $\calO(dN)$. 
Therefore, the time complexity of our algorithm mainly comes from solving the least squares problems~\eqref{e:new-lsq-3} using SGD methods in a blackbox fashion. And the time complexity for SGD methods depends on the condition number of~\eqref{e:new-lsq-3}. 
Denote
\begin{align}
  \Q_{\lambda}=  \left[
    \begin{array}{cc}
      \lambda \Sxx & - \Sxy \\
      - \Sxy^\top & \lambda \Syy
    \end{array}
  \right]=
  \left[
    \begin{array}{cc}
      \Sxx^{\frac{1}{2}} & \\
      & \Syy^{\frac{1}{2}}
    \end{array}\right]
  \left[
    \begin{array}{cc}
      \lambda \I & - \T \\
      - \T^\top & \lambda \I
    \end{array}
  \right]
  \left[
    \begin{array}{cc}
      \Sxx ^{\frac{1}{2}}& \\
      & \Syy^{\frac{1}{2}}
    \end{array}
  \right].
\end{align}
\begin{flalign*}
\text{It is clear that} \qquad\qquad \sigma_{\max}(\Q_{\lambda}) & \le (\lambda+\rho_1)\cdot \max \left( \sigma_{\max}(\Sxx), \sigma_{\max}(\Syy) \right), &  \\
 \sigma_{\min}(\Q_{\lambda}) & \ge (\lambda-\rho_1)\cdot \min \left( \sigma_{\min}(\Sxx), \sigma_{\min}(\Syy) \right). &
\end{flalign*}
We have shown in the proof of Theorem~\ref{thm:shift-and-invert-phase-I} that $\frac{\lambda + \rho_1}{\lambda - \rho_1} \le \frac{9}{\tilde{\Delta}} \le \frac{9}{c_1 {\Delta}} $ throughout Algorithm~\ref{alg:meta-shift-and-invert} (cf. Lemma~\ref{lem:shift-and-invert-repeat-until}, Appendix~\ref{append:proof-shift-and-invert-iteration-complexity}), and thus the condtion number for AGD is $\frac{\sigma_{\max} (\Q_{\lambda})}{\sigma_{\min} (\Q_{\lambda})} \le \frac{9/c_1}{\rho_1-\rho_2} \tilde{\kappa}^\prime$, where $\tilde{\kappa}^\prime := \frac{\max \left( \sigma_{\max}(\Sxx),\, \sigma_{\max}(\Syy) \right)} {\min \left( \sigma_{\min}(\Sxx),\, \sigma_{\min}(\Syy) \right) }$. 
For SVRG/ASVRG, the relevant condition number depends on the gradient Lipschitz constant of individual components. We show in Appendix~\ref{append:cond-h} (Lemma~\ref{lem:shift-and-invert-cond-svrg}) that the relevant condition number is at most $\frac{9/c_1}{\rho_1-\rho_2} \tilde{\kappa}$, where 
$\tilde{\kappa} := \frac{\max_i\, \max \left(\norm{\x_i}^2,\, \norm{\y_i}^2 \right)}{\min \left( \sigma_{\min}(\Sxx),\, \sigma_{\min}(\Syy) \right) }$. 
An interesting issue for SVRG/ASVRG is that, depending on the value of $\lambda$, the independent components $h_t^i(\uu,\vv)$ may be nonconvex. If $\lambda\ge 1$, each component is still guaranteed to by convex; otherwise, some components might be non-convex, with the overall average $\frac{1}{N} \sum_{i=1}^N \h_t^i$ being convex. In the later case, we use the modified analysis of SVRG~\cite[Appendix~B]{GarberHazan15c} for its time complexity. We use warm-start in SI as in ALS, and the initial suboptimality for each subproblem can be bounded similarly.

The total time complexities of our SI meta-algorithm are given in Table~\ref{t:final}. Note that $\tilde{\kappa}$ (or $\tilde{\kappa}^\prime$) and $\frac{1}{\rho_1-\rho_2}$ are multiplied together, giving the effective condition number. 
When using SVRG as the least squares solver, we obtain the total time complexity of $\smash{ \tilde{\calO} \left( d (N + \tilde{\kappa} \frac{1}{\rho_1 - \rho_2} ) \cdot \log^2 \left(\frac{1}{\eta}\right) \right) }$ if all components are convex, and $\tilde{\calO} \left( d (N + (\tilde{\kappa} \frac{1}{\rho_1 - \rho_2})^2 ) \cdot \log^2 \left(\frac{1}{\eta}\right) \right)$ otherwise. When using ASVRG, we have $\tilde{\calO} \left( d  \sqrt{N}  \sqrt{\tilde{\kappa}}  \sqrt{\frac{1}{\rho_1-\rho_2}} \cdot \log^2 \left(\frac{1}{\eta}\right) \right)$ if all components are convex, and $\tilde{\calO} \left( d N^{\frac{3}{4}} \sqrt{\tilde{\kappa}} \sqrt{\frac{1}{\rho_1-\rho_2}} \cdot \log^2 \left(\frac{1}{\eta}\right) \right)$ otherwise. 
Here $\tilde{\calO}(\cdot)$ hides poly-logarithmic dependences on $\frac{1}{\tilde{\mu}}$ and $\frac{1}{\Delta}$. % $\mathtt{polylog} \left(\frac{1}{\tilde{\mu}}, \frac{1}{\Delta}, \frac{1}{\eta} \right)$
It is remarkable that the SI meta-algorithm is able to separate the dependence of dataset size $N$ from other parameters in the time complexities.

\textbf{Parallel work\ } In a parallel work~\cite{Ge_16a}, the authors independently proposed a similar ALS algorithm\footnote{Our arxiv preprint for the ALS meta-algorithm was posted before their paper got accepted by ICML~2016.}, and they solve the least squares problems using AGD. The time complexity of their algorithm for extracting the first canonical correlation is $\smash{ \tilde{\calO} \left( d N \sqrt{\kappa^\prime} \, \frac{\rho_1^2}{\rho_1^2 - \rho_2^2} \cdot \log \left(\frac{1}{\eta} \right) \right) }$, 
which has linear dependence on ${ \frac{\rho_1^2}{\rho_1^2 - \rho_2^2} \log \left(\frac{1}{\eta} \right) }$ (so their algorithm is linearly convergent, but our complexity for ALS+AGD has quadratic dependence on this factor), but typically worse dependence on $N$ and $\kappa^\prime$ (see remarks in Section~\ref{sec:svrg}). Moreover, our SI algorithm tends to significantly outperform ALS theoretically and empirically. It is future work to remove extra $\log\left(\frac{1}{\eta}\right)$ dependence in our analysis.

% \subsection{Extension to multi-dimensional projections}
% \label{sec:extension}

% In practice, we may be interested in projections of more than one dimensions. 
% The CCA objective for extracting $L$-dimensional projections is
% \begin{align*}
%   \max_{\U\in \bbR^{d_x \times L}, \V\in \bbR^{d_y \times L}}\quad \trace{\U^\top \Sxy \V} 
%   \qquad \text{s.t.}\quad \U^\top \Sxx \U = \V^\top \Syy \V = \I.
% \end{align*}

\textbf{Extension to multi-dimensional projections\ } To extend our algorithms % Algorithm~\ref{alg:meta-alsvr} and~\ref{alg:meta-shift-and-invert} 
to $L$-dimensional projections, we can extract the dimensions sequentially and remove the explained correlation from $\Sxy$ each time we extract a new dimension~\cite{Witten_09a}. For the ALS meta-algorithm, a cleaner approach is to extract the $L$ dimensions simultaneously using (inexact) orthogonal iterations~\cite{GolubLoan96a}, % (an extension of power iterations to multiple dimensions, \cite{GolubLoan96a}). In this case, 
in which case the subproblems become multi-dimensional regressions and our normalization steps are of the form $\U_{t} \leftarrow \tilde{\U}_{t} ( \tilde{\U}_{t}^\top \Sxx \tilde{\U}_{t} )^{-\frac{1}{2}}$ (the same normalization is used by~\cite{Ma_15b,Wang_15c}). Such normalization involves the eigenvalue decomposition of a $L\times L$ matrix and can be solved exactly as we typically look for low dimensional projections. Our analysis for $L=1$ can be extended to this scenario and the convergence rate of ALS will depend on the gap between $\rho_{L}$ and $\rho_{L+1}$.

\vspace*{-1.5ex}
\section{Experiments}
\label{sec:expt}
\vspace*{-2ex}

We demonstrate the proposed algorithms, namely \ALSVR, \ALSAVR, \ShiftVR, and \ShiftAVR, abbreviated as ``meta-algorithm -- least squares solver'' (VR for SVRG, and AVR for ASVRG) on three real-world datasets: Mediamill~\cite{Snoek_06a} ($N=3\times 10^4$), JW11~\cite{Westbur94a} ($N=3\times 10^4$), and MNIST~\cite{Lecun_98a} ($N=6\times 10^4$). 
We compare our algorithms with batch \AppGrad\ and its stochastic version \sAppGrad~\cite{Ma_15b}, as well as the \CCALin\ algorithm in parallel work~\cite{Ge_16a}. 
For each algorithm, we compare the canonical correlation estimated by the iterates at different number of passes over the data with that of the exact solution by SVD. 
For each dataset, we vary the regularization parameters  $\gamma_x=\gamma_y$ over $\{10^{-5}, 10^{-4}, 10^{-3}, 10^{-2}\}$ to vary the least squares condition numbers, and larger regularization leads to better conditioning. We plot the suboptimality in objective vs. \# passes for each algorithm in Figure~\ref{fig:real_data}. Experimental details (e.g. SVRG parameters) are given in Appendix~\ref{append:expts}.

\begin{figure}[t]
\centering
\begin{tabular}{@{}c@{\hspace{0.005\linewidth}}c@{\hspace{0.01\linewidth}}c@{\hspace{0.005\linewidth}}c@{\hspace{0.005\linewidth}}c@{}}
& \scriptsize $\gamma_x=\gamma_y=10^{-5}$ & \scriptsize $\gamma_x=\gamma_y=10^{-4}$ & \scriptsize $\gamma_x=\gamma_y=10^{-3}$ & \scriptsize $\gamma_x=\gamma_y=10^{-2}$ \\[-.5ex]
& 
{\scriptsize $\kappa^\prime=53340,\ \delta=5.345$} & 
{\scriptsize $\kappa^\prime=5335,\ \delta=4.924$} & 
{\scriptsize $\kappa^\prime=534.4,\ \delta=4.256$} & 
{\scriptsize $\kappa^\prime=54.34,\ \delta=2.548$} \\[.5ex]
\rotatebox{90}{\hspace{1.5em}Mediamill} &
\rotatebox{90}{\hspace{1.5em}\scriptsize Suboptimality} 
\includegraphics[width=0.235\linewidth]{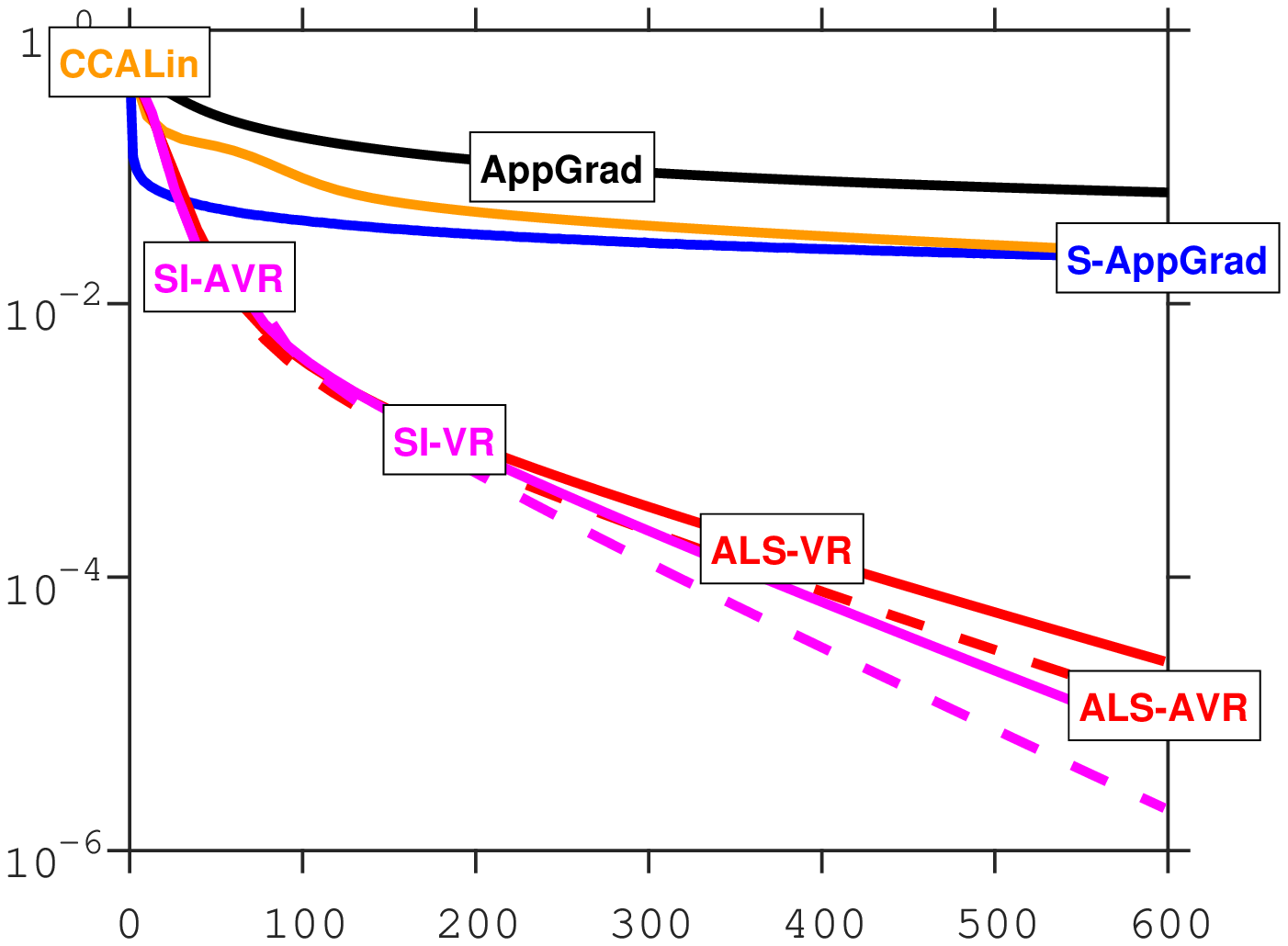} &
\includegraphics[width=0.235\linewidth]{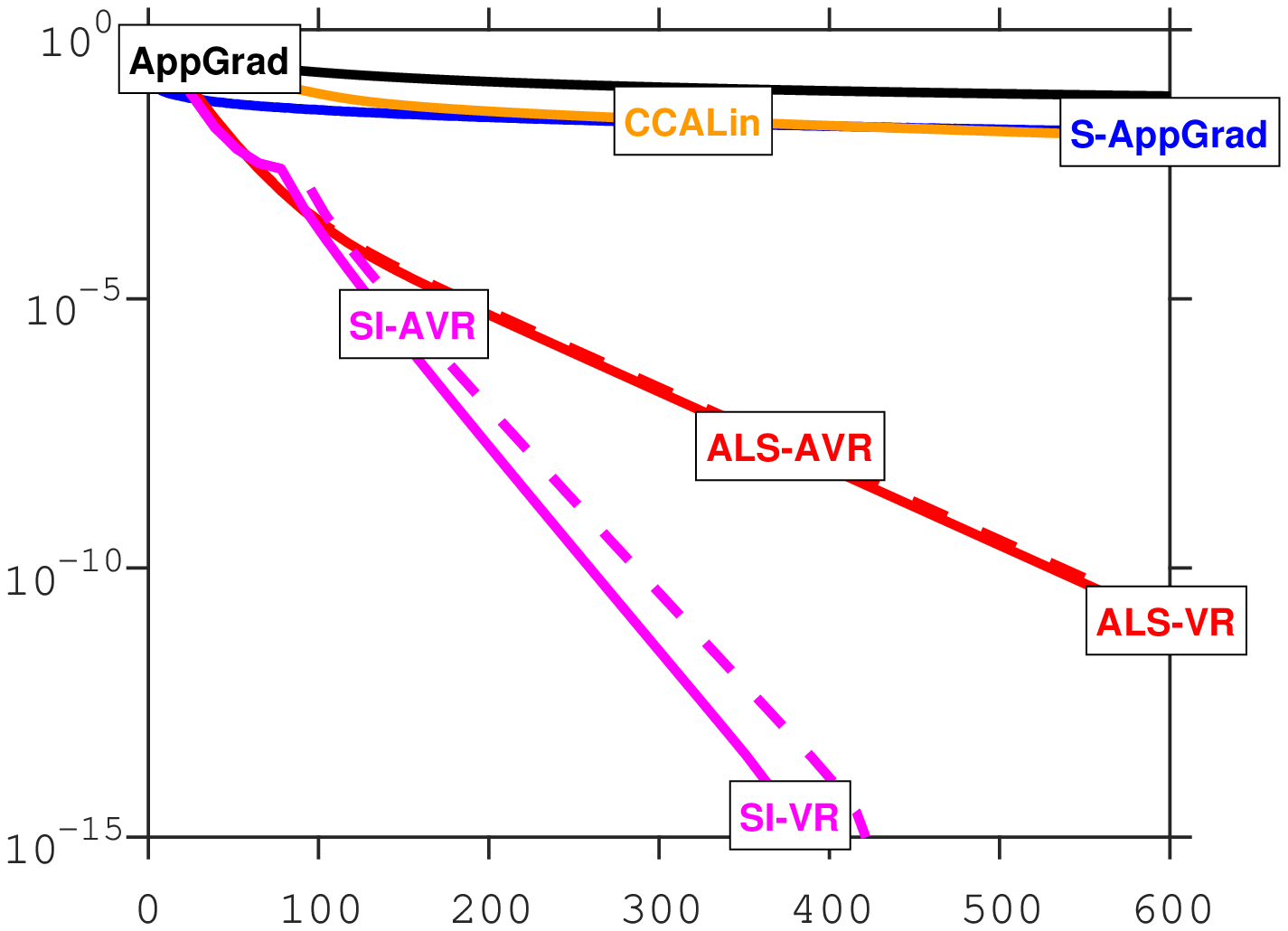} &
\includegraphics[width=0.235\linewidth]{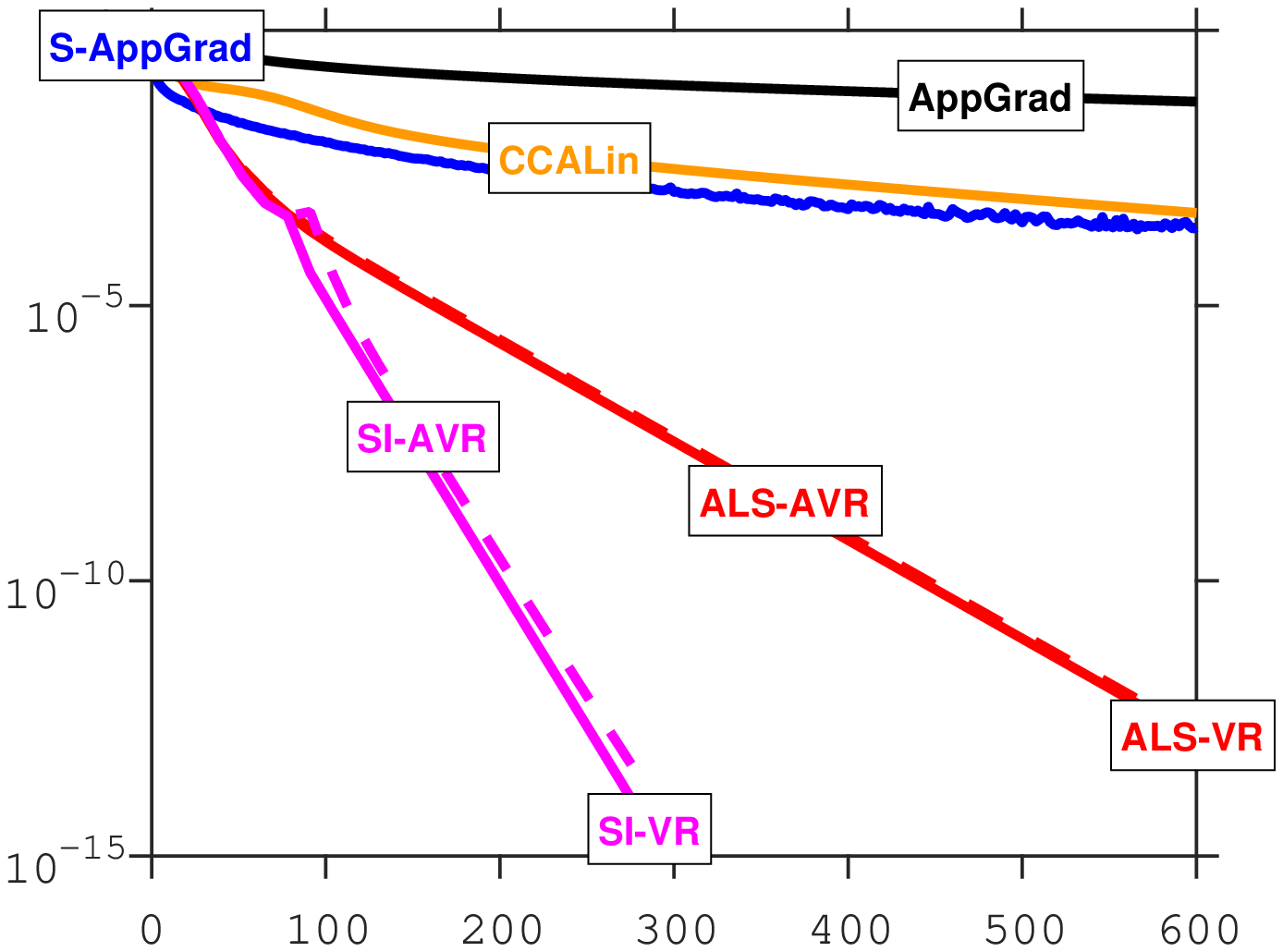} &
\includegraphics[width=0.235\linewidth]{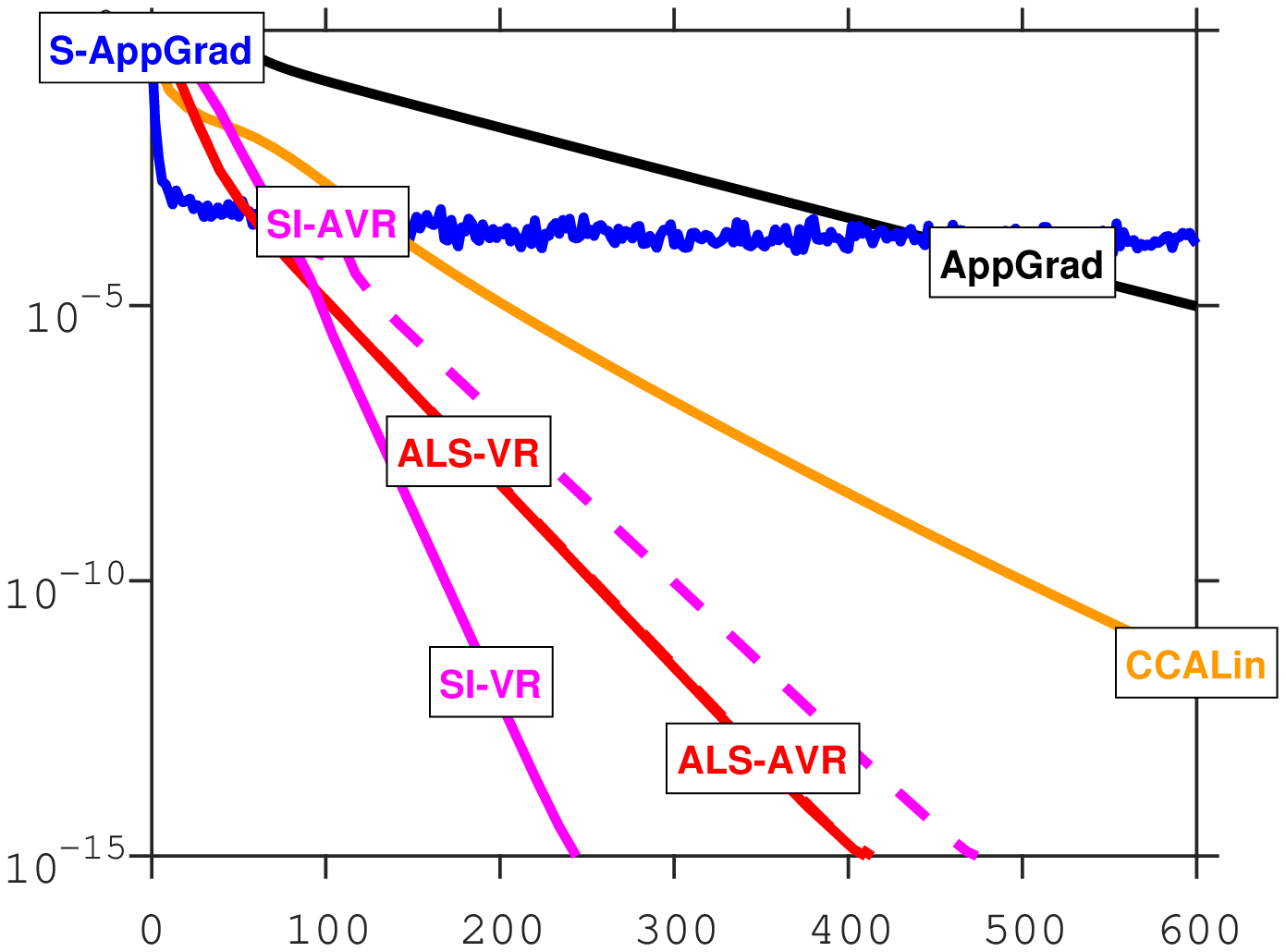} \\[-.5ex]
& 
{\scriptsize $\kappa^\prime=2699000,\ \delta=11.22$} & 
{\scriptsize $\kappa^\prime=332800,\ \delta=11.10$} & 
{\scriptsize $\kappa^\prime=34070,\ \delta=10.58$} & 
{\scriptsize $\kappa^\prime=3416,\ \delta=9.082$} \\[.0ex]
\rotatebox{90}{\hspace{2.5em}JW11} &
\rotatebox{90}{\hspace{1.5em}\scriptsize Suboptimality} 
\includegraphics[width=0.235\linewidth]{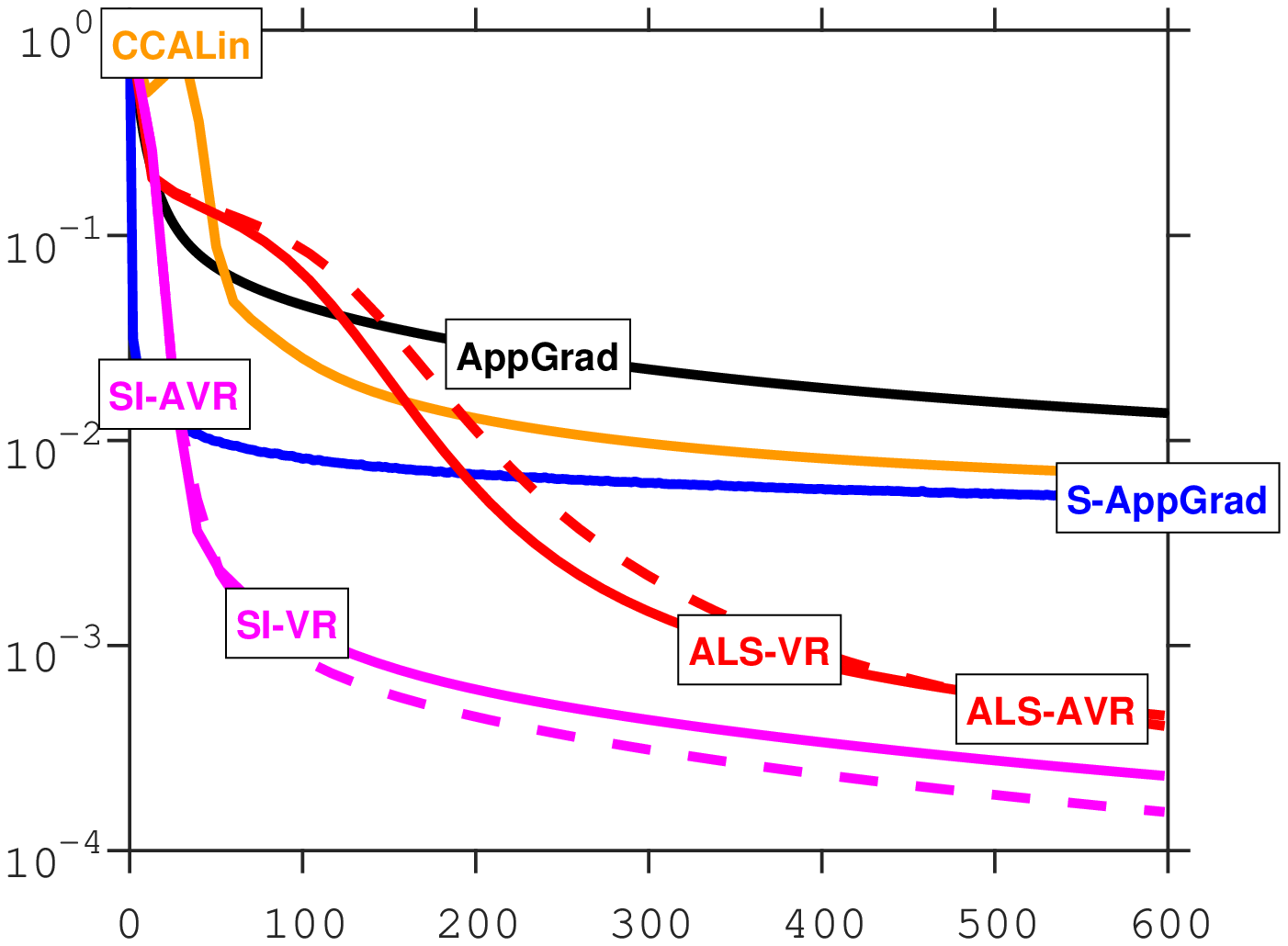} &
\includegraphics[width=0.235\linewidth]{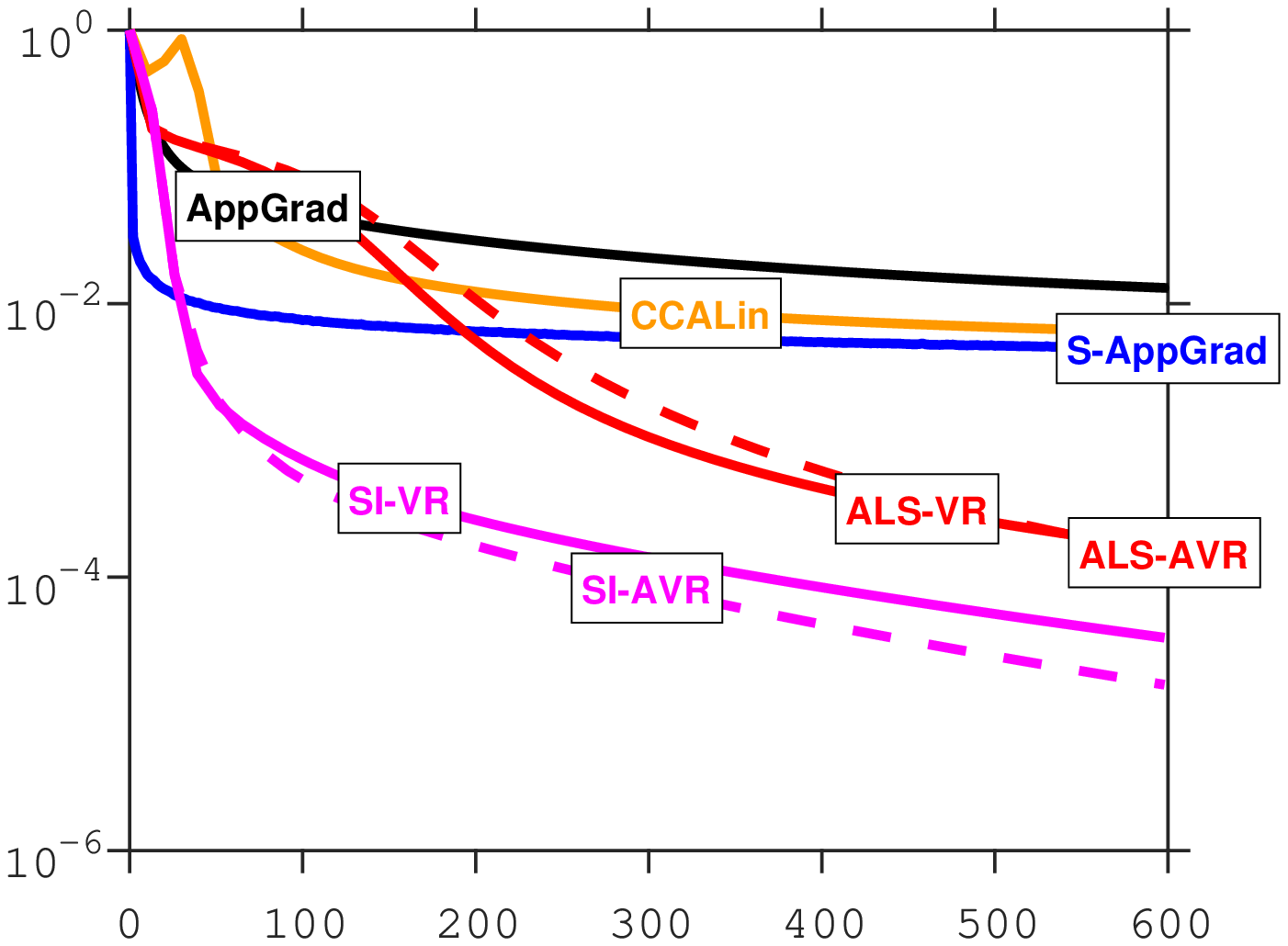} &
\includegraphics[width=0.235\linewidth]{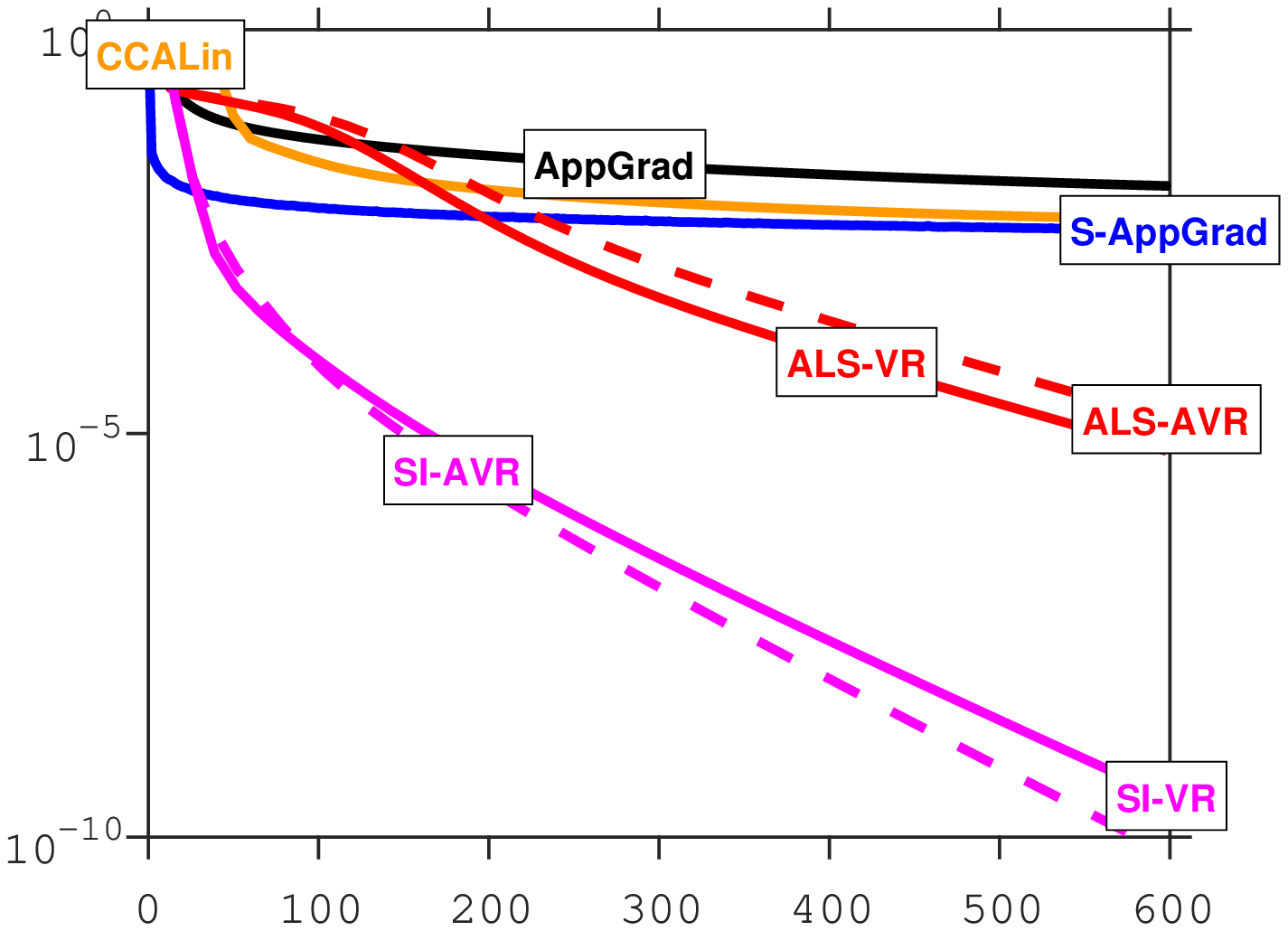} &
\includegraphics[width=0.235\linewidth]{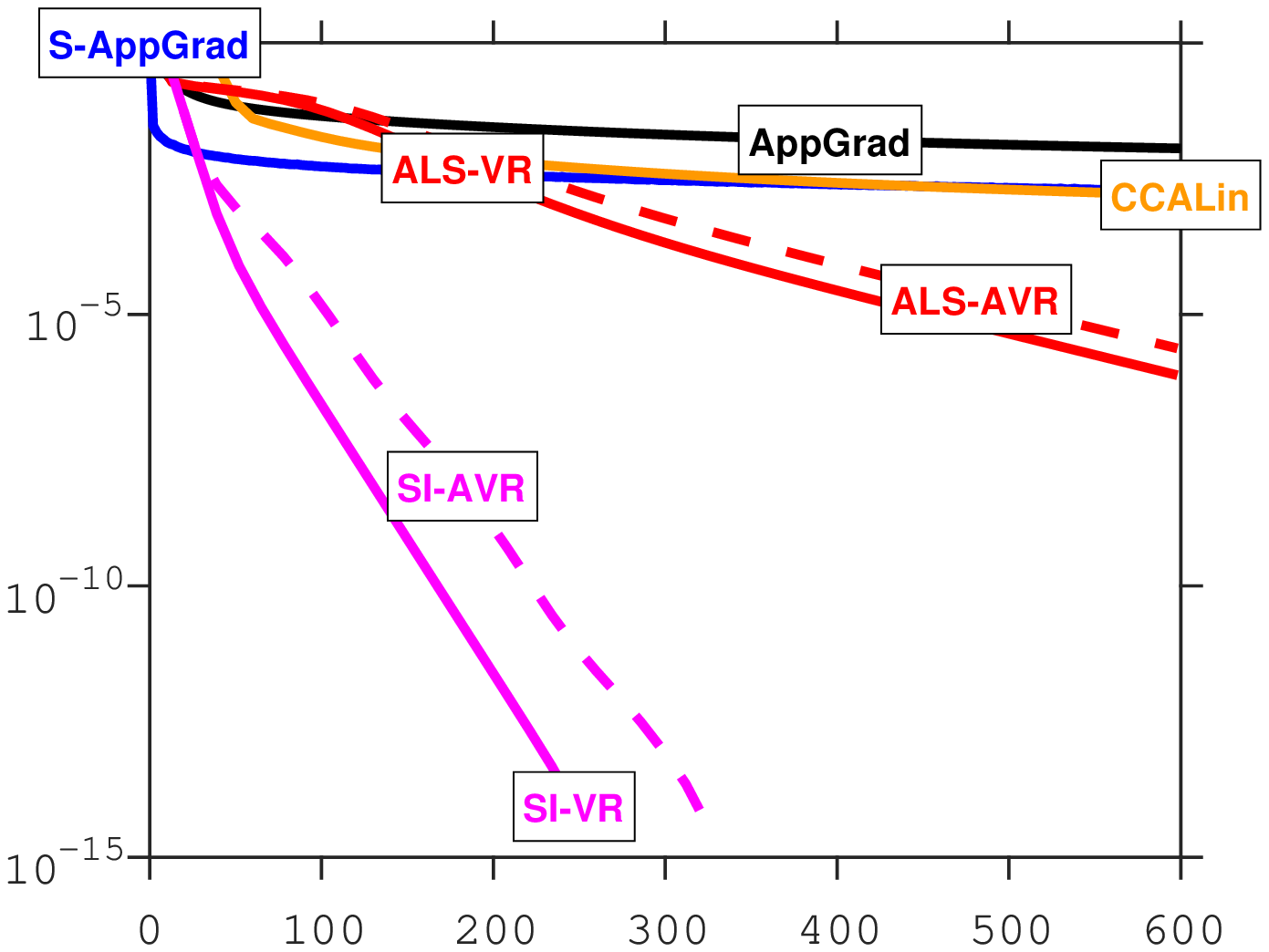} \\[-.5ex]
& 
{\scriptsize $\kappa^\prime=2235000,\ \delta=12.82$} & 
{\scriptsize $\kappa^\prime=223500,\ \delta=12.75$} & 
{\scriptsize $\kappa^\prime=22350,\ \delta=12.30$} & 
{\scriptsize $\kappa^\prime=2236,\ \delta=9.874$} \\[.5ex]
\rotatebox{90}{\hspace{1.5em}MNIST} &
\rotatebox{90}{\hspace{1.5em}\scriptsize Suboptimality} 
\includegraphics[width=0.235\linewidth]{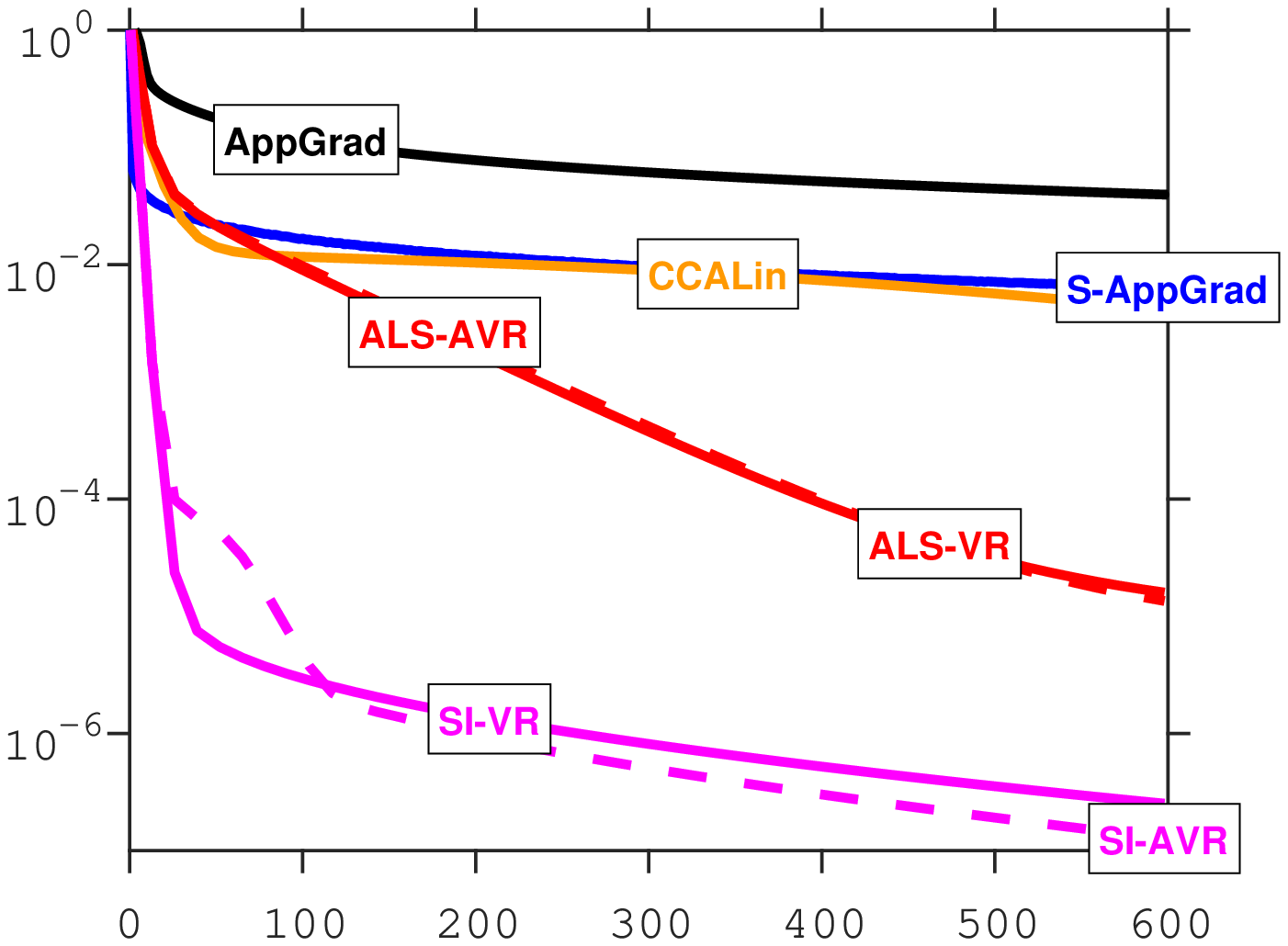} &
\includegraphics[width=0.235\linewidth]{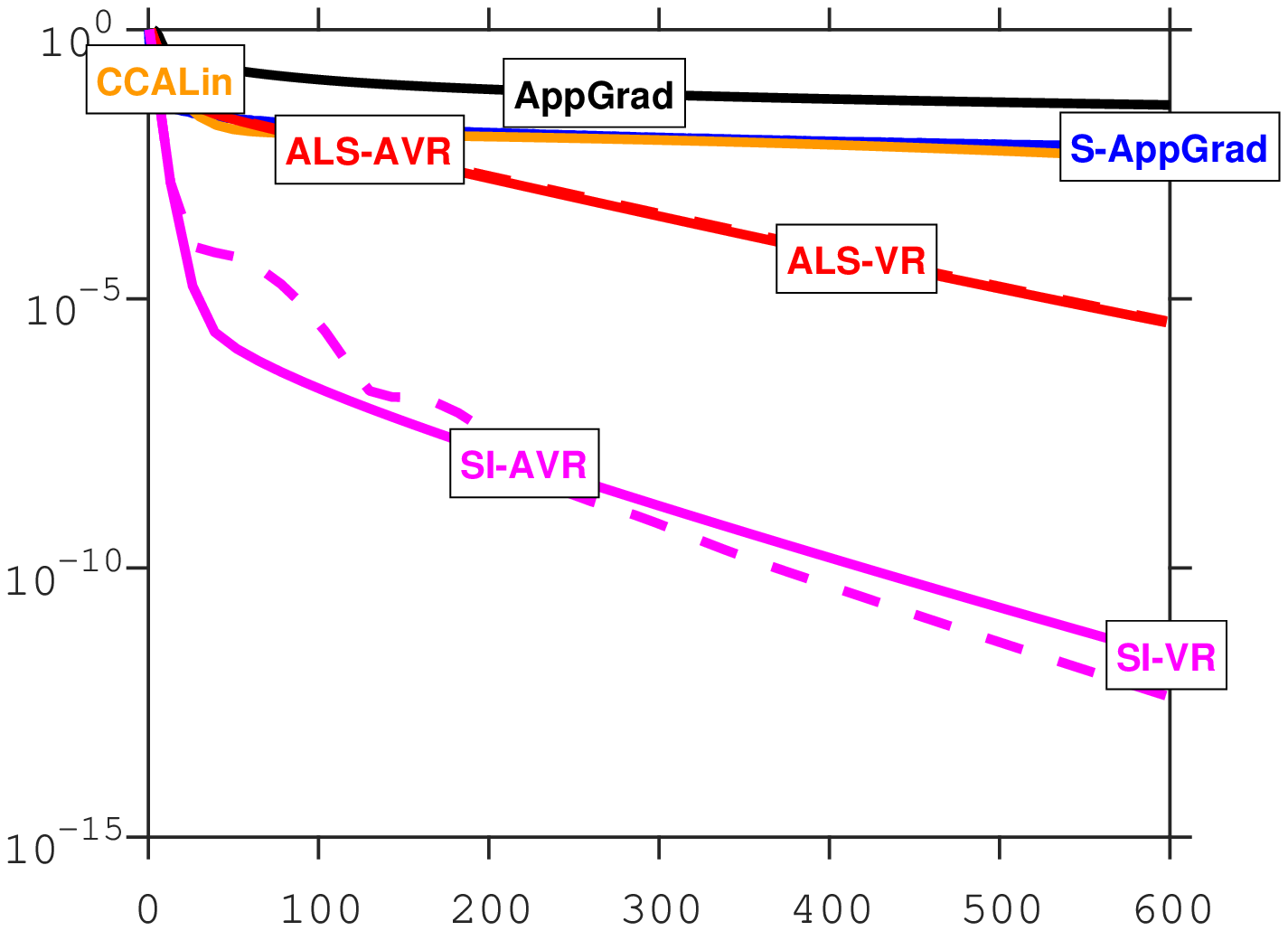} &
\includegraphics[width=0.235\linewidth]{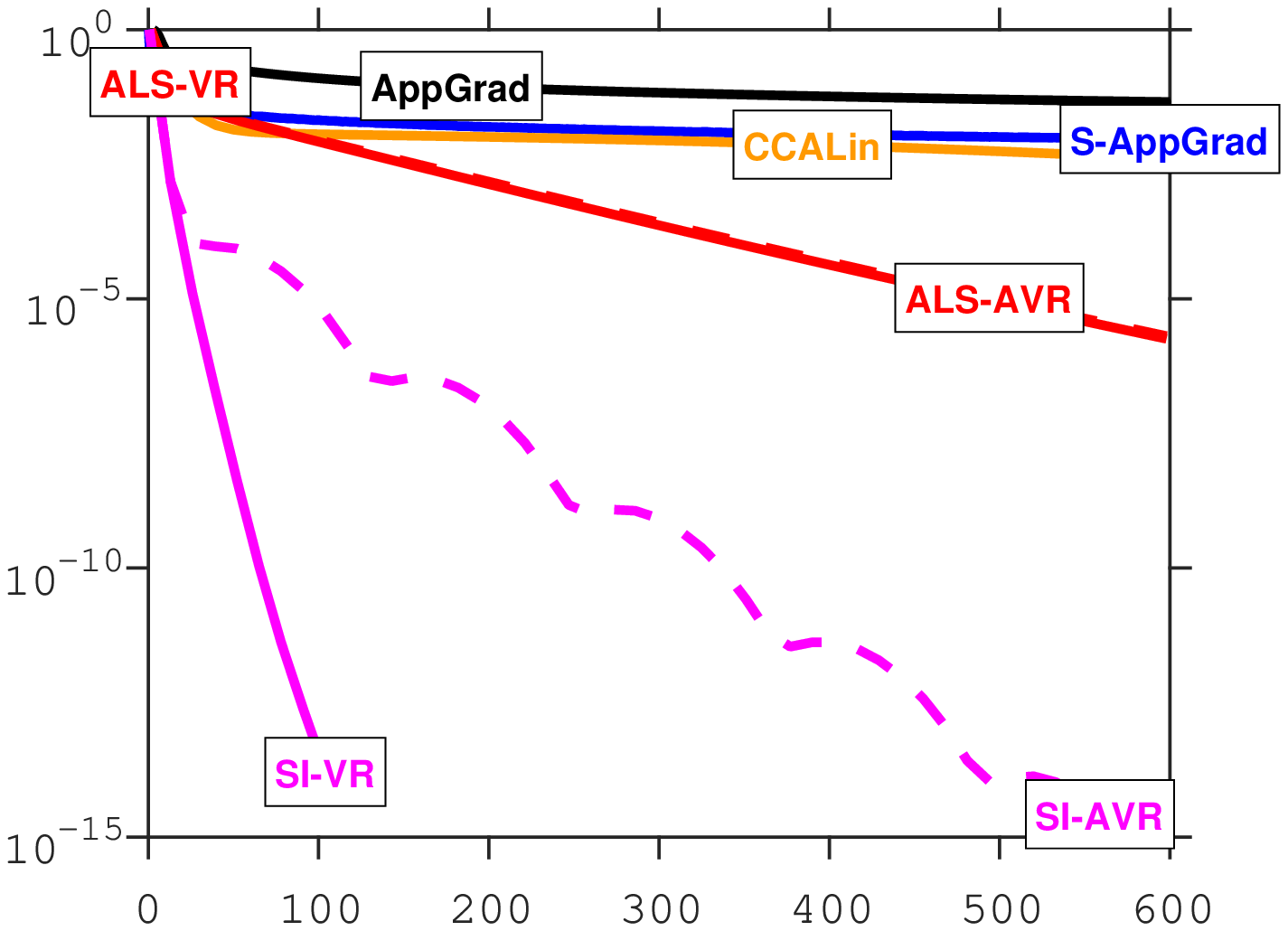} &
\includegraphics[width=0.235\linewidth]{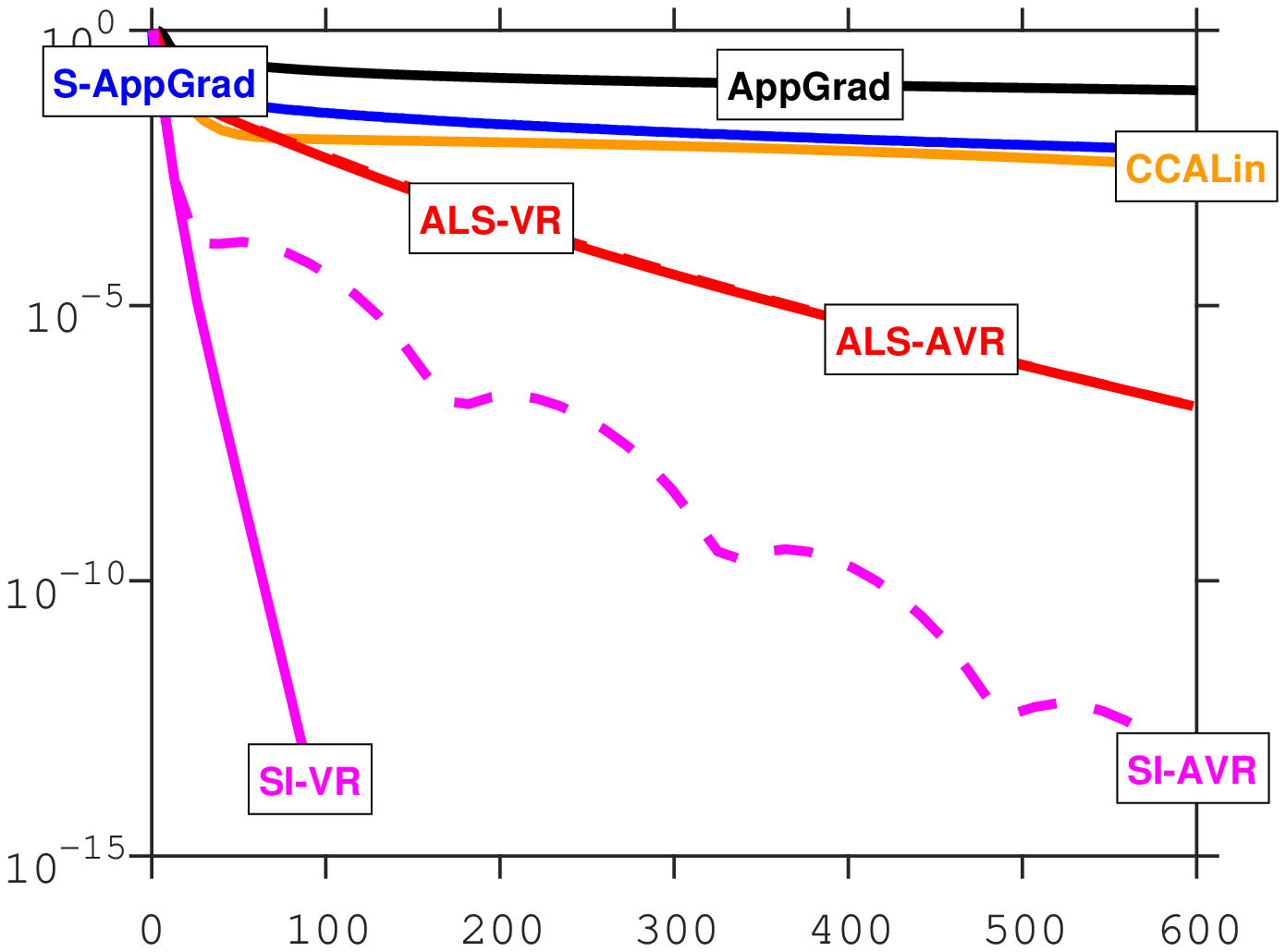} \\ [-.5ex]
& \scriptsize \# Passes & \scriptsize \# Passes & \scriptsize \# Passes & \scriptsize \# Passes
\end{tabular}\vspace*{-1ex}
\caption{Comparison of suboptimality vs. \# passes for different algorithms. For each dataset and regularization parameters $(\gamma_x,\,\gamma_y)$, we give $\kappa^\prime = \max\left(\frac{\sigma_{\max} (\Sxx)}{\sigma_{\min} (\Sxx)},\ \frac{\sigma_{\max} (\Syy)}{\sigma_{\min} (\Syy)}\right)$ and $\delta = \frac{\rho_{1}^2}{\rho_{1}^2 - \rho_{2}^2}$.}
\label{fig:real_data}
\vspace*{-2.5ex}
\end{figure}

We make the following observations from the results. First, the proposed stochastic algorithms significantly outperform batch gradient based methods \AppGrad/\CCALin. This is because  the least squares condition numbers for these datasets are large, and SVRG enable us to decouple dependences on the dataset size $N$ and the condition number $\kappa$ in the time complexity. Second, \ShiftVR\ converges faster than \ALSVR\ as it further decouples the dependence on $N$ and the singular value gap of $\T$. Third, inexact normalizations keep the \sAppGrad\ algorithm from converging to an accurate solution. Finally, ASVRG improves over SVRG when the the condition number is large.

\vspace*{-1.5ex}
\subsubsection*{Acknowledgments} \vspace*{-1.5ex}
Research partially supported by NSF BIGDATA grant 1546500. Weiran Wang would like to thank Karen Livescu for bringing him into the world of CCA. And Karen Livescu would like to thank Sham Kakade for starting her working on CCA.

\clearpage
\bibliographystyle{unsrt}
\bibliography{nips16a}

\clearpage
\appendix
\section{Proof of Theorem~\ref{thm:alternating-least-squares-exact}}
\label{append:proof-alternating-least-squares-exact}

\begin{proof}
  It is easy to see that by the end of the first iteration of Algorithm~\ref{alg:als},  $\tpsi_1$ and $\bpsi_1$ lie in the span of $\{\b_i\}_{i=1}^r$, while $\tphi_1$ and $\bphi_1$ lie in the span of $\{\aa_i\}_{i=1}^r$. And therefore they remain in these spaces for all $t\ge 1$.

  Let us first focus on $\bphi_{t}$. For $t\ge 2$, we observe that
  \begin{align*}
    \bphi_{t}=\T \bpsi_{t-1} \big/ \norm{\tphi_{t}} = \T \T^\top \bphi_{t-2} \big/ \left(\norm{\tphi_{t}}\cdot \norm{\tpsi_{t-1}} \right).
  \end{align*}
  Since $\norm{\bphi_{t-2}}=\norm{\bphi_{t}}=1$, it is equivalent to using the following updates:
  \begin{align*}
    \bphi_{t} \leftarrow \T \T^\top \bphi_{t-2}, \qquad \bphi_{t} \leftarrow \bphi_{t} / \norm{\bphi_{t}}.
  \end{align*}
  This indicates that, Algorithm~\ref{alg:als} runs the standard power iterations on $\T \T^\top$ to generate the $\{\bphi_{t}\}_{t\ge 1}$ sequence for every two steps.

  \textbf{(i)} For $t=2,4,\dots$, we have $\bphi_{t}= \frac{(\T \T^\top)^{\frac{t}{2}} \bphi_0}{\norm{(\T \T^\top)^{\frac{t}{2}} \bphi_0}}$. Let $\M=\T \T^\top$, whose nonzero eigenvalues are $\rho_1^2 \ge \rho_2^2 \ge \dots \ge \rho_r^2 >0$, with corresponding eigenvectors $\aa_{1}, \dots, \aa_r$. Then, for $i=1,\dots,r$, 
  \begin{align*}
    (\aa_{i}^\top \bphi_{t})^2 & = \frac{ \left( \aa_{i}^\top \M^{\frac{t}{2}} \bphi_0 \right)^2 }{\norm{\M^{\frac{t}{2}} \bphi_0}^2} = \frac{\left(\aa_{i}^\top \M^{\frac{t}{2}} \bphi_0 \right)^2}{\bphi_0^\top \M^{t} \bphi_0} = \frac{\left(\rho_i^{t} \aa_{i}^\top \bphi_0 \right)^2}{\sum_{j=1}^{r} \rho_j^{2t} (\aa_{j}^\top \bphi_0)^2} = \frac{\left(\aa_{i}^\top \bphi_0 \right)^2}{\sum_{j=1}^{r} \left( \frac{\rho_j^2}{\rho_i^2} \right)^{t} (\aa_{j}^\top \bphi_0)^2}  \\
    & \le \frac{\left(\aa_{i}^\top \bphi_0 \right)^2}{\left( \frac{\rho_1^2}{\rho_i^2} \right)^{t} (\aa_{1}^\top \bphi_0)^2}
    = \frac{\left(\aa_{i}^\top \bphi_0 \right)^2}{(\aa_{1}^\top \bphi_0)^2} \left( \frac{\rho_i^2}{\rho_1^2} \right)^{t}
    = \frac{\left(\aa_{i}^\top \bphi_0 \right)^2}{(\aa_{1}^\top \bphi_0)^2} \left( 1 - \frac{\rho_1^2 - \rho_i^2}{\rho_1^2} \right)^{t}  \\  
    & \le \frac{\left(\aa_{i}^\top \bphi_0 \right)^2}{(\aa_{1}^\top \bphi_0)^2} \exp\left(
      -  \frac{\rho_1^2 - \rho_i^2}{\rho_1^2} t \right).
  \end{align*}

  \textbf{(ii)} For $t=1,3,\dots$, we have $\bphi_{t} = \frac{(\T \T^\top)^{\frac{t-1}{2}} \T \bpsi_0}{\norm{(\T \T^\top)^{\frac{t-1}{2}} \T \bpsi_0}}$. Let $\N=\T^\top \T$, whose nonzero eigenvalues are $\rho_1^2 \ge \rho_2^2 \ge \dots \ge \rho_r^2 >0$, with corresponding eigenvectors $\b_{1}, \dots, \b_r$. Then, for $i=1,\dots,r$, 
  \begin{align}
    (\aa_{i}^\top \bphi_{t})^2 & = \frac{ \left( \aa_{i}^\top (\T \T^\top)^{\frac{t-1}{2}} \T \bpsi_0 \right)^2 }{\norm{(\T \T^\top)^{\frac{t-1}{2}} \T \bpsi_0}^2} 
    = \frac{\left( (\T^\top \aa_{i})^\top \N^{\frac{t-1}{2}} \bpsi_0 \right)^2}{\bpsi_0^\top \N^{t} \bpsi_0} 
    = \frac{\left(\rho_i^{t} \b_{i}^\top \bpsi_0 \right)^2}{\sum_{j=1}^{r} \rho_j^{2t} (\b_{j}^\top \bpsi_0)^2} \nonumber \\
    %% & = \frac{\left(\b_{i}^\top \bpsi_0 \right)^2}{\sum_{j=1}^{r} \left( \frac{\rho_j^2}{\rho_i^2} \right)^{t} (\b_{j}^\top \bpsi_0)^2} 
    & \le \frac{\left(\b_{i}^\top \bpsi_0 \right)^2}{(\b_{1}^\top \bpsi_0)^2} \exp\left(
      - \frac{\rho_1^2 - \rho_i^2}{\rho_1^2} t  \right).  \nonumber
  \end{align}

  Given $\delta \in (0,1)$, define $S(\delta)=\{i: \rho_i^2 > (1-\delta) \rho_1^2\}$. For $\delta_1, \delta_2 \in (0,1)$, define
  \begin{align*}
    T(\delta_1,\delta_2) := \ceil{ \frac{1}{\delta_1} \log\left( \frac{1}{\mu \delta_2} \right) }.
  \end{align*}
  For all $i \not \in S(\delta_1)$, when $t > T(\delta_1,\delta_2)$, it holds that $(\aa_i^\top \bphi_{t})^2 \le \delta_2 (\aa_{i}^\top \bphi_0)^2$ if $t$ is even, and  $(\aa_i^\top \bphi_{t})^2 \le \delta_2 (\b_{i}^\top \bpsi_0)^2$ if $t$ is odd. In both cases, we have $\sum_{i\in S(\delta_1)} (\aa_{i}^\top \bphi_{t})^2 \ge 1-\delta_2$.

  When there exists a postive singular value gap, \ie, $\rho_1-\rho_2>0$, set $\delta_1=(\rho_1^2-\rho_2^2)/\rho_1^2$ and thus $S(\delta_1)={1}$. Futhermore, set $\delta_2=\eta$ and we obtain $(\aa_{1}^\top \bphi_{t})^2 \ge 1 - \eta$.

  The proof for $\bpsi_{t}$ is completely analogous. To obtain the bound on the objective, we have
  \begin{align*}
    \uu_{t}^\top \Sxy \vv_{t} & = \bphi_{t}^\top \T \bpsi_{t} = \rho_1 (\bphi_{t}^\top \aa_1) (\bpsi_{t}^\top \b_1) + \sum_{i=2}^{r} \rho_i (\bphi_{t}^\top \aa_i) (\bpsi_{t}^\top \b_i) \\
    & \ge \rho_1 (\bphi_{t}^\top \aa_1) (\bpsi_{t}^\top \b_1) - \rho_1 \sum_{i=2}^{r} \abs{\bphi_{t}^\top \aa_i}  \abs{\bpsi_{t}^\top \b_i} \\
    & \ge \rho_1 (1-\eta) - \rho_1 \sqrt{\sum_{i=2}^r \left( \bphi_{t}^\top \aa_i\right)^2} \sqrt{\sum_{i=2}^r \left( \bpsi_{t}^\top \b_i \right)^2} \\
    & \ge \rho_1 (1-\eta) - \rho_1 \eta = \rho_1 (1 - 2 \eta),
  \end{align*}
  where we have used the Cauchy-Schwarz inequality in the second inequality.
\end{proof}

\section{Proof of Theorem~\ref{thm:alternating-least-squares-inexact}}
\label{append:proof-alternating-least-squares-inexact}

From now on, we distinguish the iterates of our stochastic algorithm (Algorithm~\ref{alg:meta-alsvr}) from the iterates of the exact power iterations (Algorithm~\ref{alg:als}) and denote the latter with asterisks, \ie, $\tu_{t}^*$ and $\tv_{t}^*$ for the unnormalized iterates and $\uu_{t}^*$ and $\vv_{t}^*$ for the normalized iterates. We denote the exact optimum of $f_t (\uu)$ and $g_t (\vv)$ by $\bu_{t}$ and $\bv_{t}$ respectively.

The following lemma bounds the distance between the iterates of inexact and exact power iterations.
\begin{lemma} \label{lem:alternating-least-squares-distance}
  Assume that Algorithm~\ref{alg:als} and Algorithm~\ref{alg:meta-alsvr} start with the same initialization, \ie, $\tu_0=\tu^*_0$ and $\tv_0=\tv^*_0$. 
  Then, for $t\ge 1$, the unnormalized iterates of Algorithm~\ref{alg:meta-alsvr} satisfy
  \begin{align*}
    \max\left(\norm{\Sxx^{\frac{1}{2}}\tu_{t} - \Sxx^{\frac{1}{2}}\tu_{t}^*},\ \norm{\Syy^{\frac{1}{2}}\tv_{t} - \Syy^{\frac{1}{2}}\tv_{t}^*} \right) \le \tilde{S}_t,
  \end{align*}
  where
  \begin{align*}
    \tilde{S}_t := \sqrt{2 \epsilon}\,\frac{ (2\rho_1/\rho_r)^{t}-1}{(2\rho_1/\rho_r) -1 }.
  \end{align*}
  Furthermore, for $t\ge 1$, the normalized iterates  of Algorithm~\ref{alg:meta-alsvr} satisfy
  \begin{align*}
    \max\left(\norm{\Sxx^{\frac{1}{2}}\uu_{t} - \Sxx^{\frac{1}{2}}\uu_{t}^*},\ \norm{\Syy^{\frac{1}{2}}\vv_{t} - \Syy^{\frac{1}{2}}\vv_{t}^*} \right) \le S_t := \frac{2 \tilde{S}_t}{\rho_r}.
  \end{align*}
\end{lemma}
\begin{proof}
  We focus on the $\{\tu_{t}\}_{t\ge 0}$ and $\{\uu_{t}\}_{t\ge 0}$ sequences below; the proof for $\{\tv_{t}\}_{t\ge 0}$ and $\{\vv_{t}\}_{t\ge 0}$ is completely analogous. 

We prove the bound for unnormalized iterates by induction. First, the case for $t=1$ holds trivially. For $t \ge 2$, we can bound the error of the unnormalized iterates using the exact solution to $f_t(\uu)$: 
  \begin{align} \label{e:split-errors-1}
    \norm{\Sxx^{\frac{1}{2}}\tu_{t} - \Sxx^{\frac{1}{2}}\tu_{t}^*} \le 
    \norm{\Sxx^{\frac{1}{2}}\tu_{t} - \Sxx^{\frac{1}{2}}\bu_{t}} + 
    \norm{\Sxx^{\frac{1}{2}}\bu_{t} - \Sxx^{\frac{1}{2}}\tu_{t}^*} .
  \end{align}

  For the first term of~\eqref{e:split-errors-1}, notice $f_t(\uu)$ is a quadratic function with minimum achieved at $\bu_{t}=\Sxx^{-1} \Sxy \vv_{t-1}$. For the approximate solution $\tu_{t}$, we have 
  \begin{align*}
    f_t(\tu_{t}) - f_t(\bu_{t}) = \frac{1}{2} (\tu_{t} - \bu_{t})^\top  \Sxx (\tu_{t} - \bu_{t}) = \frac{1}{2} \norm{\Sxx^{\frac{1}{2}} \tu_{t} - \Sxx^{\frac{1}{2}} \bu_{t} }^2 \le \epsilon.
  \end{align*}
  It then follows that $\norm{\Sxx^{\frac{1}{2}}\tu_{t} - \Sxx^{\frac{1}{2}}\bu_{t}} \le \sqrt{2\epsilon}$.

  The second term of~\eqref{e:split-errors-1} is concerned with the error due to inexact target in the least squares problem $f_t(\uu)$ as $\vv_{t-1}$ is different from $\vv_{t-1}^*$. We can bound it as 
  \begin{align}
    \norm{ \Sxx^{\frac{1}{2}} \bu_{t} - \Sxx^{\frac{1}{2}} \tu_{t}^* }
    & = \norm{ \Sxx^{\frac{1}{2}} \Sxx^{-1} \Sxy \vv_{t-1} - \Sxx^{\frac{1}{2}} \Sxx^{-1} \Sxy \vv_{t-1}^* } \nonumber \\
    & = \norm{ \left( \Sxx^{-\frac{1}{2}} \Sxy \Syy^{-\frac{1}{2}} \right) \left( \Syy^{\frac{1}{2}} (\vv_{t-1} - \vv_{t-1}^*) \right) }  \nonumber \\ \label{e:split-errors-2}
    & \le  \norm{ \T } \norm{\Syy^{\frac{1}{2}} \vv_{t-1} - \Syy^{\frac{1}{2}} \vv_{t-1}^*}  = \rho_1 \norm{\Syy^{\frac{1}{2}} \vv_{t-1} - \Syy^{\frac{1}{2}} \vv_{t-1}^*}.
  \end{align}

In view of the update rule of our algorithm and the triangle inequality, we have 
  \begin{align} 
    & \norm{\Syy^{\frac{1}{2}} \vv_{t-1} - \Syy^{\frac{1}{2}} \vv_{t-1}^* } \nonumber \\
    \le & \norm{\frac{\Syy^{\frac{1}{2}}\tv_{t-1}}{\norm{\Syy^{\frac{1}{2}}\tv_{t-1}}} - \frac{\Syy^{\frac{1}{2}}\tv_{t-1}}{\norm{\Syy^{\frac{1}{2}}\tv_{t-1}^*}} } + \norm{\frac{\Syy^{\frac{1}{2}}\tv_{t-1}}{\norm{\Syy^{\frac{1}{2}}\tv_{t-1}^*}} - \frac{\Syy^{\frac{1}{2}}\tv_{t-1}^*}{\norm{\Syy^{\frac{1}{2}}\tv_{t-1}^*}} } \nonumber \\
    = & \norm{\Syy^{\frac{1}{2}}\tv_{t-1}} \abs{ \frac{1}{\norm{\Syy^{\frac{1}{2}}\tv_{t-1}}} - \frac{1}{\norm{\Syy^{\frac{1}{2}}\tv_{t-1}^*}} }
    + \frac{1}{\norm{\Syy^{\frac{1}{2}}\tv_{t-1}^*}} \norm{\Syy^{\frac{1}{2}}\tv_{t-1} - \Syy^{\frac{1}{2}}\tv_{t-1}^*}  \nonumber \\
    = & \frac{1}{\norm{\Syy^{\frac{1}{2}}\tv_{t-1}^*}} \abs{ \norm{\Syy^{\frac{1}{2}}\tv_{t-1}^*} - \norm{\Syy^{\frac{1}{2}}\tv_{t-1}} } + \frac{1}{\norm{\Syy^{\frac{1}{2}}\tv_{t-1}^*}} \norm{\Syy^{\frac{1}{2}}\tv_{t-1} - \Syy^{\frac{1}{2}}\tv_{t-1}^*}  \nonumber \\ \label{e:split-errors-3}
    \le & \frac{2}{\norm{\Syy^{\frac{1}{2}}\tv_{t-1}^*}} \norm{\Syy^{\frac{1}{2}}\tv_{t-1} - \Syy^{\frac{1}{2}}\tv_{t-1}^*}  
    \le \frac{2 \tilde{S}_{t-1}}{\norm{\Syy^{\frac{1}{2}}\tv_{t-1}^*}} .
  \end{align}
  We now bound $\norm{\Syy^{\frac{1}{2}}\tv_{t-1}^*}$ from below. Since $t\ge 2$, we have
  \begin{align*}
    \Syy^{\frac{1}{2}} \tv_{t-1}^* = \Syy^{\frac{1}{2}} \Syy^{-1} \Sxy^\top \uu_{t-2}^* = \left( \Syy^{-\frac{1}{2}} \Sxy^\top \Sxx^{-\frac{1}{2}} \right) \left( \Sxx^{\frac{1}{2}} \uu_{t-2}^* \right) = \T^\top \left( \Sxx^{\frac{1}{2}} \uu_{t-2}^* \right).
  \end{align*}
  Now, $\Sxx^{\frac{1}{2}} \uu_{t-2}^*$ corresponds to $\bphi_{t-2}$ in Algorithm~\ref{alg:als}, which has unit length and lies in the span of $\{\aa_1,\dots,\aa_r\}$, so we have
\begin{align*}
    \norm{\Syy^{\frac{1}{2}} \tv_{t-1}^*} = \norm{\T^\top \bphi_{t-2}} \ge \rho_r .
  \end{align*}
Combining~\eqref{e:split-errors-1},~\eqref{e:split-errors-2} and~\eqref{e:split-errors-3} gives
\begin{align*}
  \norm{\Sxx^{\frac{1}{2}}\tu_{t} - \Sxx^{\frac{1}{2}}\tu_{t}^*} 
& \le \sqrt{2 \epsilon} + \frac{2 \rho_1}{\rho_r} \cdot \tilde{S}_{t-1}
= \sqrt{2 \epsilon} + \frac{2 \rho_1}{\rho_r} \cdot \sqrt{2 \epsilon} \frac{ (2\rho_1/\rho_r)^{t-1}-1}{ (2\rho_1/\rho_r) -1 } \\
& = \sqrt{2 \epsilon}\,\frac{ (2\rho_1/\rho_r)^{t}-1}{(2\rho_1/\rho_r) -1 } = \tilde{S}_t.
\end{align*}
The bound for normalized iterates follows from~\eqref{e:split-errors-3}.
\end{proof}

\begin{proof}[\textbf{Proof of Theorem~\ref{thm:alternating-least-squares-inexact}}]
  We prove the theorem by relating the iterates of inexact power iterations to those of exact power iterations. 

  Assume the same initialization as in Lemma~\ref{lem:alternating-least-squares-distance}. First observe that
  \begin{align}
    (\uu_{t}^\top \Sxx \uu^*)^2 & = \left( \left(\uu_{t}^*\right)^\top \Sxx \uu^* + \left( \uu_{t} - \uu_{t}^* \right)^\top \Sxx \uu^* \right)^2  \nonumber \\
    & \ge \left( \left(\uu_{t}^*\right)^\top \Sxx \uu^*\right)^2 + 2 \left( \left(\uu_{t}^*\right)^\top \Sxx \uu^* \right) \left( \left(\uu_{t}-\uu_{t}^*\right)^\top  \Sxx \uu^* \right) \nonumber \\
    & \ge \left( \left(\uu_{t}^*\right)^\top \Sxx \uu^*\right)^2 - 2 \abs{ \left(\Sxx^{\frac{1}{2}} \left(\uu_{t}-\uu_{t}^*\right) \right)^\top \left(\Sxx^{\frac{1}{2}} \uu^* \right) } \nonumber \\ \label{e:inexact-last-step}
    & \ge \left( \left(\uu_{t}^*\right)^\top \Sxx \uu^*\right)^2 - 2 \norm{\Sxx^{\frac{1}{2}} \uu_{t}-\Sxx^{\frac{1}{2}} \uu_{t}^*}
  \end{align}
  where we have used the fact that $\norm{\Sxx^{\frac{1}{2}} \uu_{t}}=\norm{\Sxx^{\frac{1}{2}} \uu_{t}^*}=\norm{\Sxx^{\frac{1}{2}} \uu^*}=1$ and the Cauchy-Schwarz inequality in the last two steps.

  Applying Theorem~\ref{thm:alternating-least-squares-exact} with $T \ge \ceil{ \frac{\rho_1^2}{\rho_1^2-\rho_2^2} \log\left( \frac{2}{\mu \eta} \right) }$, we have that $\left( \left(\uu_T^*\right)^\top \Sxx \uu^*\right)^2 \ge 1-\eta/2$. On the other hand, in view of Lemma~\ref{lem:alternating-least-squares-distance}, we have for the specified $\epsilon$ value in Algorithm~\ref{alg:meta-alsvr} that $\norm{\Sxx^{\frac{1}{2}} \uu_{T}-\Sxx^{\frac{1}{2}} \uu_{T}^*} \le S_T=\eta/4$. Plugging these two bounds into \eqref{e:inexact-last-step} gives the desired result.

  The proof for $\vv_T$ is completely analogous.
\end{proof}

\section{SVRG for minimizing $f (\uu)$}
\label{append:alg-svrg}

We provide the pseudo-code of SVRG for solving the least squares problem~\eqref{e:lsq} below. % in Algorithm~\ref{alg:svrg}.

\begin{algorithm*}[h]
  \caption*{SVRG for $\min_{\uu}\ f(\uu):= \frac{1}{N} \sum_{i=1}^N  \left( \frac{1}{2} \abs{\uu^\top \x_i - \vv^\top \y_i }^2 + \frac{\gamma_x}{2} \norm{\uu}^2 \right)$.}
  % \label{alg:svrg}
  \renewcommand{\algorithmicrequire}{\textbf{Input:}}
  \renewcommand{\algorithmicensure}{\textbf{Output:}}
  \begin{algorithmic}
    \REQUIRE Stepsize $\xi$.
    \STATE Initialize $\uu_{(0)} \in \bbR^{d_x}$.
    \FOR{$j=1,2,\dots,M$}
    \STATE $\w_{0} \leftarrow \uu_{(j-1)}$
    \STATE Evaluate the batch gradient $\nabla f (\w_{0}) = \X ( \X^\top \w_{0} - \Y^\top \vv) / N + \gamma_x \w_{0}$
    \FOR{$t=1,2,\dots,m$}
    \STATE Randomly pick $i_t$ from $\{1,\dots,N\}$
    \STATE $\w_{t} \leftarrow \w_{t-1} - \xi \left( ( \x_{i_t} \x_{i_t}^\top + \gamma_x \I) (\w_{t-1}-\w_{0}) + \nabla f (\w_{0}) \right)$ 
    \ENDFOR
    \STATE $\uu_{(j)} \leftarrow \w_{t}$ for randomly chosen $t\in \{1,\dots,m\}$.
    \ENDFOR
    \ENSURE $\uu_{(M)}$ is the approximate solution.
  \end{algorithmic}
\end{algorithm*}

\section{Initial suboptimality of warm-starts in Algorithm~\ref{alg:meta-alsvr}}
\label{append:meta-alternating-least-squares-initial-suboptimality}

At time step $t$, we initialize the least squares problem $f_t(\uu)$ with the unnormalized iterate $\tu_{t-1}$ from the previous time step. We now bound the suboptimality of this initialization. Observe that the minimum of $f_t(\uu)$ is achieved by $\bu_{t} = \Sxx^{-1} \Sxy \vv_{t-1}$, and that
\begin{align*}
  f_t(\tu_{t-1}) - f_t(\bu_{t}) & = \frac{1}{2} (\tu_{t-1} - \bu_{t})^\top \Sxx (\tu_{t-1} - \bu_{t}) = \frac{1}{2} \norm{\Sxx^{\frac{1}{2}}\tu_{t-1} - \Sxx^{\frac{1}{2}} \bu_{t} }^2 .
\end{align*}
Applying the triangle inequality, we have for $t=1$ that
\begin{align*}
\norm{\Sxx^{\frac{1}{2}}\tu_{0} - \Sxx^{\frac{1}{2}} \bu_{1} } & \le
\norm{\Sxx^{\frac{1}{2}}\tu_{0}} + \norm{ \Sxx^{\frac{1}{2}} \bu_{1} }
\le 1 + \norm{ \Sxx^{\frac{1}{2}} \Sxx^{-1} \Sxy \vv_{0} } \\
& = 1 + \norm{ \T \Syy^{\frac{1}{2}} \vv_{0}} 
\le 1 + \norm{\T} \norm{\Syy^{\frac{1}{2}} \vv_{0}}
= 1 + \rho_1 \le 2
\end{align*}
where we have used facts that $\norm{ \Syy^{\frac{1}{2}} \tu_{0} } = \norm{ \Syy^{\frac{1}{2}} \vv_{0} } = 1$ due to the initial normalizations.

And we have for $t \ge 2$ that
\begin{align*}
\norm{\Sxx^{\frac{1}{2}}\tu_{t-1} - \Sxx^{\frac{1}{2}} \bu_{t} } & \le
\norm{\Sxx^{\frac{1}{2}}\tu_{t-1} - \Sxx^{\frac{1}{2}} \bu_{t-1} } + 
\norm{\Sxx^{\frac{1}{2}}\bu_{t-1} - \Sxx^{\frac{1}{2}} \bu_{t} } \\
& \le \sqrt{2 \epsilon} + \norm{\Sxx^{\frac{1}{2}} \Sxx^{-1} \Sxy \vv_{t-2} - \Sxx^{\frac{1}{2}} \Sxx^{-1} \Sxy \vv_{t-1} } \\
& = \sqrt{2 \epsilon} + \norm{ \T \left(\Syy^{\frac{1}{2}} \vv_{t-2} - \Syy^{\frac{1}{2}} \vv_{t-1} \right) } \\
& \le \sqrt{2 \epsilon} + \norm{\T} \norm{ \Syy^{\frac{1}{2}} \vv_{t-2} - \Syy^{\frac{1}{2}} \vv_{t-1} } \\
& \le \sqrt{2 \epsilon} + 2 \rho_1 \le \sqrt{2 \epsilon} + 2
\end{align*}
where we have used the fact that $\norm{ \Syy^{\frac{1}{2}} \vv_{t-2} } = \norm{ \Syy^{\frac{1}{2}} \vv_{t-1} }=1$ in the last inequality.

Therefore, for all $t \ge 1$, the ration between initial suboptimality and required accuracy is 
\begin{align*}
  \frac{f_t(\tu_{t-1}) - f_t(\bu_{t})}{\epsilon} \sim \frac{2}{\epsilon}.
\end{align*}

\clearpage
\section{The shift-and-invert preconditioning (SI) algorithm for CCA}
\label{append:alg-shift-and-invert}

Our shift-and-invert preconditioning (SI) meta-algorithm is detailed in Algorithm~\ref{alg:meta-shift-and-invert}.

\begin{algorithm}[h]
  \caption{The shift-and-invert preconditioning meta-algorithm for CCA.}
  \label{alg:meta-shift-and-invert}
  \renewcommand{\algorithmicrequire}{\textbf{Input:}}
  \renewcommand{\algorithmicensure}{\textbf{Output:}}
  \begin{algorithmic}
    \REQUIRE Data matrices $\X$, $\Y$, regularization parameters $(\gamma_x, \gamma_y)$, an estimate $\tilde{\Delta}$ for $\Delta=\rho_1 - \rho_2$.
    \STATE Initialize $\tu_0 \in \bbR^{d_x},\ \tv_0 \in \bbR^{d_y}$
    \STATE $\uu_0 \leftarrow \tu_0 \big/ \sqrt{\tu_0^\top \Sxx \tu_0},\qquad \vv_0 \leftarrow \tv_0 \big/ \sqrt{\tv_0^\top \Syy \tv_0}$
    \STATE \textbf{// Phase I: shift-and-invert preconditioning for eigenvectors of $\M_{\lambda}$}
    \STATE $s \leftarrow 0, \qquad \lambda_{(0)} \leftarrow 1 + \tilde{\Delta}$
    \REPEAT 
    \STATE $s \leftarrow s+1$
    \FOR{$t=(s-1)  m_1+1, \dots, s m_1$}
    \STATE Optimize the least squares problem
    \begin{align*}
      \min_{\uu,\vv}\ h_t(\uu,\vv) := \frac{1}{2}
      \left[\uu^\top \vv^\top \right]
      \left[
        \begin{array}{cc}
          \lambda_{(s-1)} \Sxx & - \Sxy \\
          - \Sxy^\top & \lambda_{(s-1)} \Syy
        \end{array}
      \right]
      \left[
        \begin{array}{c}
          \uu \\ \vv
        \end{array}
      \right]
      - \uu^\top \Sxx \uu_{t-1} - \vv^\top \Syy \vv_{t-1}
    \end{align*}
    and output an approximate solution $(\tu_t, \tv_t)$ satisfying $h_t (\tu_t,\tv_t) \le \min_{\uu,\vv} h_t (\uu,\vv) + \te$.
    \STATE Normalization: 
    $\left[ \begin{array}{c} \uu_{t} \\ \vv_{t} \end{array} \right] \leftarrow
    \sqrt{2} \left[ \begin{array}{c} \tu_{t} \\ \tv_{t} \end{array} \right] \bigg/ 
    \sqrt{ \tu_{t}^\top \Sxx \tu_{t} + \tv_{t}^\top \Syy \tv_{t} }$
    \ENDFOR
    \STATE Optimize the least squares problem
    \begin{align*}
      \min_{\w}\ l_s (\w) := \frac{1}{2}
      \w^\top
      \left[
        \begin{array}{cc}
          \lambda_{(s-1)} \Sxx & - \Sxy \\
          - \Sxy^\top & \lambda_{(s-1)} \Syy
        \end{array}
      \right]
      \w
      -  \w^\top \left[ \begin{array}{c} \Sxx \uu_{s m_1} \\ \Syy \vv_{s m_1}\end{array} \right]
    \end{align*}
    and output an approximate solution $\w_s$ satisfying $l_s (\w_s) \le \min_{\w} l_s (\w) + \te$.
    \STATE $\Delta_s \leftarrow \frac{1}{2} \cdot \frac{1}{\frac{1}{2} \left[\uu_{s m_1}^\top \vv_{s m_1}^\top\right] \left[ \begin{array}{cc} \Sxx & \\ & \Syy \end{array}\right] \w_s - 2 \sqrt{\te/\tilde{\Delta}}}\,, \qquad 
    \lambda_{(s)} \leftarrow \lambda_{(s-1)} - \frac{\Delta_s}{2}$
    \UNTIL {$\Delta_{(s)} \le \tilde{\Delta}}$
    \STATE $\lambda_{(f)} \leftarrow \lambda_{(s)}$
    \FOR{$t = s m_1 + 1, s m_1 + 2 , \dots, s m_1 + m_2$}
    \STATE Optimize the least squares problem
    \begin{align*}
      \min_{\uu,\vv}\ h_t(\uu,\vv) := \frac{1}{2}
      \left[\uu^\top \vv^\top \right]
      \left[
        \begin{array}{cc}
          \lambda_{(f)} \Sxx & - \Sxy \\
          - \Sxy^\top & \lambda_{(f)} \Syy
        \end{array}
      \right]
      \left[
        \begin{array}{c}
          \uu \\ \vv
        \end{array}
      \right]
      - \uu^\top \Sxx \uu_{t-1} - \vv^\top \Syy \vv_{t-1}
    \end{align*}
    and output an approximate solution $(\tu_t, \tv_t)$ satisfying $h_t (\tu_t,\tv_t) \le \min_{\uu,\vv} h_t (\uu,\vv) + \te$.
    \STATE Normalization: 
    $\left[ \begin{array}{c} \uu_{t} \\ \vv_{t} \end{array} \right] \leftarrow
    \sqrt{2} \left[ \begin{array}{c} \tu_{t} \\ \tv_{t} \end{array} \right] \bigg/ 
    \sqrt{ \tu_{t}^\top \Sxx \tu_{t} + \tv_{t}^\top \Syy \tv_{t} }$
    \ENDFOR
    \STATE \textbf{// Phase II: Final normalization}
    \STATE $T \leftarrow s m_1+m_2, \qquad \hu \leftarrow \uu_{T} / \sqrt{\uu_{T}^\top \Sxx \uu_{T}}, \qquad \hv \leftarrow \vv_{T} / \sqrt{\vv_{T}^\top \Syy \vv_{T}} $
    \ENSURE $(\hu,\hv)$ is the approximate solution to CCA.
  \end{algorithmic}
\end{algorithm}

\section{Proof of Theorem~\ref{thm:shift-and-invert-phase-I}}
\label{append:proof-shift-and-invert-phase-I}

The proof of Theorem~\ref{thm:shift-and-invert-phase-I} closely follows that of~\cite[Theorem~4.2]{GarberHazan15c}. And we will need a few lemmas on the convergence of inexact power iterations.

\subsection{Auxiliary lemmas}

Define the condition number of $\M_{\lambda}$ as
\begin{align*}
  \kappa_{\lambda} : = \frac{\sigma_1(\M_{\lambda})}{\sigma_d(\M_{\lambda})} = \frac{\frac{1}{\lambda-\rho_1}}{\frac{1}{\lambda+\rho_1}} = \frac{\lambda+\rho_1}{\lambda-\rho_1},
\end{align*}
and the inverse relative spectral gap of $\M_{\lambda}$ as 
\begin{align*}
  \delta_{\lambda} := \frac{\sigma_1 (\M_{\lambda})}{ \sigma_1 (\M_{\lambda}) - \sigma_2 (\M_{\lambda}) } 
  = \frac{\frac{1}{\lambda - \rho_1}}{\frac{1}{\lambda - \rho_1}  - \frac{1}{\lambda-\rho_2} }
  = \frac{\lambda - \rho_2}{\rho_1 - \rho_2}.
\end{align*}

The first lemma states the convergence of exact power iterations, paralleling~\cite[Theorem~A.1]{GarberHazan15c}. % and our Theorem~\ref{thm:alternating-least-squares-exact}. 

\begin{lemma}[Convergence of exact power iterations] \label{lem:shift-and-invert-exact}
  Fix $\alpha > 0$. For the exact power iterations on $\M_{\lambda}$ where
  \begin{align*}
    \left[
      \begin{array}{c}\tu_t^* \\ \tv_t^* \end{array}
    \right]
    & \leftarrow 
    \left[
      \begin{array}{cc}
        \lambda \Sxx & - \Sxy \\
        - \Sxy^\top & \lambda \Syy
      \end{array}
    \right]^{-1}
    \left[
      \begin{array}{cc}
        \Sxx &  \\
        & \Syy
      \end{array}
    \right]
    \left[
      \begin{array}{c}\uu_{t-1}^* \\ \vv_{t-1}^* \end{array}
    \right],
    \\
    \left[ \begin{array}{c} \uu_{t}^* \\ \vv_{t}^* \end{array} \right] & \leftarrow
    \sqrt{2} \left[ \begin{array}{c} \tu_{t}^* \\ \tv_{t}^* \end{array} \right] \bigg/ 
    \sqrt{ (\tu_{t}^*)^\top \Sxx \tu_{t}^* + (\tv_{t}^*)^\top \Syy \tv_{t}^* }, 
    \qquad \text{for}\quad t=1,\dots,m,
  \end{align*}
  and $\mu^\prime := \frac{1}{4} \left( (\uu_0^*)^\top \Sxx \uu^* + (\vv_0^*)^\top \Syy \vv^* \right)^2 > 0$, we have 
  \begin{itemize}
  \item (crude regime) \begin{align*}
      \frac{1}{2} \left[ (\uu_{t}^*)^\top \Sxx^{\frac{1}{2}},\  (\vv_{t}^*)^\top \Syy^{\frac{1}{2}} \right]  \M_{\lambda} \left[\begin{array}{c} \Sxx^{\frac{1}{2}} \uu_{t}^* \\ \Syy^{\frac{1}{2}} \vv_{t}^* \end{array} \right]   \ge (1 - \alpha) \cdot \sigma_1 (\M_{\lambda})
    \end{align*}
    for $t \ge \ceil{ \frac{1}{\alpha} \log \left( \frac{2}{\mu^\prime \alpha} \right) }$,  

  \item (accurate regime) \begin{align*}
      \frac{1}{4} \left( (\uu_{t}^*)^\top \Sxx \uu^* + (\vv_{t}^*)^\top \Syy \vv^* \right)^2 \ge 1 - \alpha
    \end{align*}
    for $t \ge \ceil{ \frac{\delta_{\lambda}}{2}  \log \left( \frac{1}{\mu^\prime \alpha} \right) }$.
  \end{itemize}
\end{lemma}

The second lemma bounds the distances between the iterates of inexact and exact power iterations, paralleling~\cite[Lemma~4.1]{GarberHazan15c}. % and our Lemma~\ref{lem:alternating-least-squares-distance}. 
Recall that the $(\tu_t, \tv_t)$ in Algorithm~\ref{alg:meta-shift-and-invert} satisfies $h_t (\tu_t,\tv_t) \le \min_{\uu,\vv} h_t (\uu,\vv) + \te$. Let $(\bu_t, \bv_t)$ be the exact minimum of $h_t$. Then we have
\begin{align}
  & h_t (\tu_t,\tv_t) - h_t (\bu_t,\bv_t) \nonumber \\
  = & \frac{1}{2}
  \left[(\tu_t-\bu_t)^\top\ \ (\tv_t-\bv_t)^\top \right]
  \left[
    \begin{array}{cc}
      \lambda \Sxx & - \Sxy \\
      - \Sxy^\top & \lambda \Syy
    \end{array}
  \right]
  \left[
    \begin{array}{c}
      \tu_t-\bu_t \\ \tv_t-\bv_t
    \end{array}
  \right] \nonumber \\
  = & \frac{1}{2}
  \left[(\tu_t-\bu_t)^\top\ \ (\tv_t-\bv_t)^\top \right]
  \left[
    \begin{array}{cc}
      \lambda \Sxx & - \Sxy \\
      - \Sxy^\top & \lambda \Syy
    \end{array}
  \right]
  \left[
    \begin{array}{c}
      \tu_t-\bu_t \\ \tv_t-\bv_t
    \end{array}
  \right] \nonumber \\ 
  = & \frac{1}{2}
  \left[ (\tu_t-\bu_t)^\top \Sxx^{\frac{1}{2}}\ \ (\tv_t-\bv_t)^\top \Syy^{\frac{1}{2}} \right]
  \left[
    \begin{array}{cc}
      \lambda \I & - \T \\
      - \T^\top & \lambda \I
    \end{array}
  \right]
  \left[ \begin{array}{c}
      \Sxx^{\frac{1}{2}} (\tu_t-\bu_t) \\ \Syy^{\frac{1}{2}} (\tv_t-\bv_t)
    \end{array} \right] \nonumber \\ \label{e:SI-bound-distance}
  = & \frac{1}{2}
  \left[ (\tu_t-\bu_t)^\top \Sxx^{\frac{1}{2}}\ \ (\tv_t-\bv_t)^\top \Syy^{\frac{1}{2}} \right]
  \M_{\lambda}^{-1}
  \left[ \begin{array}{c}
      \Sxx^{\frac{1}{2}} (\tu_t-\bu_t) \\ \Syy^{\frac{1}{2}} (\tv_t-\bv_t)
    \end{array} \right] \le \te.
\end{align}

\begin{lemma} [Power iterations with inexact matrix-vector multiplications] \label{lem:shift-and-invert-inexact-distance}
  Consider the inexact power iterations on $\M_{\lambda}$ where
  \begin{align*}
    (\tu_t, \tv_t) & \qquad \text{satisfies}\qquad~\eqref{e:SI-bound-distance}, \\
    \left[ \begin{array}{c} \uu_{t} \\ \vv_{t} \end{array} \right] & \leftarrow
    \sqrt{2} \left[ \begin{array}{c} \tu_{t} \\ \tv_{t} \end{array} \right] \bigg/ 
    \sqrt{ \tu_{t}^\top \Sxx \tu_{t} + \tv_{t}^\top \Syy \tv_{t} }, 
    \qquad \text{for}\quad t=1,\dots,m.
  \end{align*}
  Compare these iterates with those of the exact power iterations described in Lemma~\ref{lem:shift-and-invert-exact} using the same initialization $\tu_0=\tu_0^*$, $\tv_0=\tv_0^*$.  Then, for $t \ge 0$, the unnormalized iterates satisfy
  \begin{align*}
    \norm{
      \frac{1}{\sqrt{2}} \left[ \begin{array}{c}
          \Sxx^{\frac{1}{2}} \tu_t \\ \Syy^{\frac{1}{2}} \tv_t
        \end{array} \right] - 
      \frac{1}{\sqrt{2}} \left[ \begin{array}{c}
          \Sxx^{\frac{1}{2}} \tu_t^* \\ \Syy^{\frac{1}{2}} \tv_t^*
        \end{array} \right]
    } \le \tilde{R}_t
  \end{align*}
  where  
  \begin{align*}
    \tilde{R}_t := \sqrt{\sigma_1 (\M_{\lambda}) \cdot \te} \cdot \frac{ \left(2 \kappa_{\lambda}\right)^t -1 }{ 2 \kappa_{\lambda} - 1 },
  \end{align*}
  while the normalized iterates satisfy
  \begin{align*}
    \norm{
      \frac{1}{\sqrt{2}} \left[ \begin{array}{c}
          \Sxx^{\frac{1}{2}} \uu_t \\ \Syy^{\frac{1}{2}} \vv_t
        \end{array} \right] - 
      \frac{1}{\sqrt{2}} \left[ \begin{array}{c}
          \Sxx^{\frac{1}{2}} \uu_t^* \\ \Syy^{\frac{1}{2}} \vv_t^*
        \end{array} \right]
    } \le {R}_t := \frac{2 \tilde{R}_t}{\sigma_d (\M_{\lambda})} .
  \end{align*}
\end{lemma}

The third lemma states the convergence of inexact power iterations, paralleling~\cite[Theorem~4.1]{GarberHazan15c}. % and our Theorem~\ref{thm:alternating-least-squares-inexact}.

\begin{lemma}[Convergence of inexact power iterations] \label{lem:shift-and-invert-inexact}
  Fix $\alpha > 0$. Consider the inexact power iterations described in Lemma~\ref{lem:shift-and-invert-inexact-distance}.
  \begin{itemize}
  \item (crude regime) Let $t_1 = \ceil{ \frac{2}{\alpha} \log \left( \frac{4}{\mu^\prime \alpha} \right) }$. Fix $T \ge t_1$, and set $\te (T) = \frac{\alpha^2 \cdot \sigma_d (\M_{\lambda}) }{64 \kappa_{\lambda}} \left( \frac{ 2 \kappa_{\lambda} - 1 }{ \left(2 \kappa_{\lambda} \right)^T -1 }\right)^2 $. Then we have
    \begin{align*}
      \frac{1}{2} \left[ \uu_T^\top \Sxx^{\frac{1}{2}},\  \vv_T^\top \Syy^{\frac{1}{2}} \right]  \M_{\lambda} \left[\begin{array}{c} \Sxx^{\frac{1}{2}} \uu_T \\ \Syy^{\frac{1}{2}} \vv_T \end{array} \right]   \ge (1 - \alpha) \cdot \sigma_1 (\M_{\lambda}).
    \end{align*}

  \item (accurate regime) Let $t_2 = \ceil{ \frac{\delta(\M_{\lambda})}{2}  \log \left( \frac{2}{\mu^\prime \alpha} \right) }$. Fix $T \ge t_2$, and set $\te(T) = \frac{\alpha^2 \cdot \sigma_d (\M_{\lambda}) }{64 \kappa_{\lambda}} \left( \frac{ 2 \kappa_{\lambda} - 1 }{ \left(2 \kappa_{\lambda} \right)^T -1 }\right)^2 $. Then we have
    \begin{align*}
      \frac{1}{4} \left( \uu_T^\top \Sxx \uu^* + \vv_T^\top \Syy \vv^* \right)^2 \ge 1 - \alpha .
    \end{align*}
  \end{itemize}
\end{lemma}

For brevity, let us define the following short-hands:
\begin{align*}
  \tilde{\rr}_t & = \frac{1}{\sqrt{2}} \left[ \begin{array}{c} \Sxx^{\frac{1}{2}} \tu_{t} \\ \Syy^{\frac{1}{2}} \tv_{t} \end{array} \right], \qquad
  {\rr}_t = \frac{1}{\sqrt{2}} \left[ \begin{array}{c} \Sxx^{\frac{1}{2}} \uu_{t} \\ \Syy^{\frac{1}{2}} \vv_{t} \end{array} \right], \qquad
  \bar{\rr}_t = \frac{1}{\sqrt{2}} \left[ \begin{array}{c} \Sxx^{\frac{1}{2}} \bu_{t} \\ \Syy^{\frac{1}{2}} \bv_{t} \end{array} \right], \\
  \tilde{\rr}_t^* & = \frac{1}{\sqrt{2}} \left[ \begin{array}{c} \Sxx^{\frac{1}{2}} \tu_{t}^* \\ \Syy^{\frac{1}{2}} \tv_{t}^* \end{array} \right], \qquad
  {\rr}_t^* = \frac{1}{\sqrt{2}} \left[ \begin{array}{c} \Sxx^{\frac{1}{2}} \uu_{t}^* \\ \Syy^{\frac{1}{2}} \vv_{t}^* \end{array} \right], \qquad
  {\rr}^* = \frac{1}{\sqrt{2}} \left[ \begin{array}{c} \Sxx^{\frac{1}{2}} \uu^* \\ \Syy^{\frac{1}{2}} \vv^* \end{array} \right].
\end{align*}
All these vectors are in $\bbR^{d}$ and have length $1$.

Observe that the matrix-vector multiplication~\eqref{e:new-lsq} is equivalent to
\begin{align*} % \label{e:new-lsq-understand}
  & \left[
    \begin{array}{c} \Sxx^{\frac{1}{2}} \tu_t \\ \Syy^{\frac{1}{2}} \tv_t \end{array}
  \right]
  \leftarrow 
  \left[
    \begin{array}{cc}
      \Sxx^{\frac{1}{2}} &  \\
      & \Syy^{\frac{1}{2}}
    \end{array}
  \right]
  \left[
    \begin{array}{cc}
      \lambda \Sxx & - \Sxy \\
      - \Sxy^\top & \lambda \Syy
    \end{array}
  \right]^{-1}
  \left[
    \begin{array}{cc}
      \Sxx^{\frac{1}{2}} &  \\
      & \Syy^{\frac{1}{2}}
    \end{array}
  \right]
  \left[
    \begin{array}{c} \Sxx^{\frac{1}{2}} \uu_{t-1} \\ \Syy^{\frac{1}{2}} \vv_{t-1} \end{array}
  \right],
\end{align*}
and 
\begin{align*}
  & \left[
    \begin{array}{cc}
      \Sxx^{\frac{1}{2}} &  \\
      & \Syy^{\frac{1}{2}}
    \end{array}
  \right]
  \left[
    \begin{array}{cc}
      \lambda \Sxx & - \Sxy \\
      - \Sxy^\top & \lambda \Syy
    \end{array}
  \right]^{-1}
  \left[
    \begin{array}{cc}
      \Sxx^{\frac{1}{2}} &  \\
      & \Syy^{\frac{1}{2}}
    \end{array}
  \right] \\
  = &  \left[
    \begin{array}{cc}
      \Sxx^{-\frac{1}{2}} &  \\
      & \Syy^{-\frac{1}{2}}
    \end{array}
  \right]^{-1}
  \left[
    \begin{array}{cc}
      \lambda \Sxx & - \Sxy \\
      - \Sxy^\top & \lambda \Syy
    \end{array}
  \right]^{-1}
  \left[
    \begin{array}{cc}
      \Sxx^{-\frac{1}{2}} &  \\
      & \Syy^{-\frac{1}{2}}
    \end{array}
  \right]^{-1} \\
  = & \left( \left[
      \begin{array}{cc}
        \Sxx^{-\frac{1}{2}} &  \\
        & \Syy^{-\frac{1}{2}}
      \end{array}
    \right]
    \left[
      \begin{array}{cc}
        \lambda \Sxx & - \Sxy \\
        - \Sxy^\top & \lambda \Syy
      \end{array}
    \right]
    \left[
      \begin{array}{cc}
        \Sxx^{-\frac{1}{2}} &  \\
        & \Syy^{-\frac{1}{2}}
      \end{array}
    \right] \right)^{-1} \\
  = & \left[
    \begin{array}{cc}
      \lambda \I & - \Sxx^{-\frac{1}{2}} \Sxy \Syy^{-\frac{1}{2}} \\
      - \Syy^{-\frac{1}{2}} \Sxy^\top \Sxx^{-\frac{1}{2}} & \lambda \I
    \end{array}
  \right]^{-1} \\
  =& \M_{\lambda} .
\end{align*}
Then the updates for exact power iterations can be written as
\begin{align*}
  \tilde{\rr}_t^* \leftarrow \M_{\lambda} \rr_{t-1}^*, \qquad {\rr}_t^* \leftarrow \tilde{\rr}_t^* / \norm{\tilde{\rr}_t^*}, \qquad t=1,\dots,
\end{align*}
and the updates for inexact power iterations can be written as
\begin{align*}
  \tilde{\rr}_t \approx \M_{\lambda} \rr_{t-1}, \qquad {\rr}_t \leftarrow \tilde{\rr}_t / \norm{\tilde{\rr}_t}, \qquad t=1,\dots.
\end{align*}

Note we have according to~\eqref{e:SI-bound-distance} that
\begin{align*}
  \te \ge 
  (\tilde{\rr}_t - \bar{\rr}_t)^\top \M_{\lambda}^{-1} (\tilde{\rr}_t - \bar{\rr}_t) \ge 
  \sigma_d (\M_{\lambda}^{-1}) \cdot \norm{ \tilde{\rr}_t - \bar{\rr}_t }^2 =
  \frac{1}{\sigma_1 (\M_{\lambda})} \cdot \norm{ \tilde{\rr}_t - \bar{\rr}_t }^2
\end{align*}
or equivalently 
\begin{align} \label{e:SI-bound-distance-rr}
  \norm{ \tilde{\rr}_t - \bar{\rr}_t } \le 
  \sqrt{\sigma_1 (\M_{\lambda}) \cdot \epsilon}.
\end{align}

\begin{proof}[\textbf{Proof of Lemma~\ref{lem:shift-and-invert-exact}}]
  Recall that the eigenvectors of $\M_{\lambda}$ are:
  \begin{align*}
    \lambda_1 := \frac{1}{\lambda-\rho_1} > \lambda_2 := \frac{1}{\lambda-\rho_2} \ge \dots \ge \lambda_{d-1} := \frac{1}{\lambda+\rho_2} \ge \lambda_d := \frac{1}{\lambda+\rho_1},
  \end{align*}
  with corresponding eigenvectors 
  \begin{align*}
    \e_1 = \rr^* = \frac{1}{\sqrt{2}}
    \left[ \begin{array}{c} \aa_1\\ \b_1 \end{array} \right], \ 
    \e_2 = \frac{1}{\sqrt{2}} 
    \left[ \begin{array}{c} \aa_2\\ \b_2 \end{array} \right], \ \dots,\ 
    \e_{d-1} = 
    \frac{1}{\sqrt{2}} 
    \left[ \begin{array}{c} \aa_2\\ - \b_2 \end{array} \right], \ 
    \e_d = 
    \frac{1}{\sqrt{2}} 
    \left[ \begin{array}{c} \aa_1\\ - \b_1 \end{array} \right].
  \end{align*}

  By the update rule of exact power iterations, it holds that for $i=1,\dots,d$ that 
  \begin{align*}
    (\e_i^\top \rr_t^*)^2 & = \frac{ \left( \e_i^\top \M_{\lambda}^t \rr_0^* \right)^2 }{ \norm{\M_{\lambda}^t \rr_0^* }^2 } 
    = \frac{ \left( \e_i^\top \M_{\lambda}^t \rr_0 \right)^2 }{ (\rr_0^*)^\top \M_{\lambda}^{2t} \rr_0^* } = \frac{ \left(\lambda_i^t \e_i^\top \rr_0^* \right)^2 }{ \sum_{j=1}^d \lambda_j^{2t} \left( \e_j^\top \rr_0^* \right)^2 } 
    = \frac{ \left( \e_i^\top \rr_0^* \right)^2 }{ \sum_{j=1}^d \left( \frac{\lambda_j}{\lambda_i} \right) ^{2t} \left( \e_j^\top \rr_0^* \right)^2 } \\
    & \le \frac{ \left( \e_i^\top \rr_0^* \right)^2 }{ \left( \frac{\lambda_1}{\lambda_i} \right) ^{2t} \left( \e_1^\top \rr_0^* \right)^2 } 
    = \frac{ \left( \e_i^\top \rr_0^* \right)^2 }{ \left( \e_1^\top \rr_0^* \right)^2 } \left( \frac{\lambda_i}{\lambda_1} \right)^{2t} 
    = \frac{ \left( \e_i^\top \rr_0^* \right)^2 }{ \tilde{\mu} } \left( 1 - \frac{\lambda_1 - \lambda_i}{\lambda_1} \right)^{2t} \\
    & \le \frac{ \left( \e_i^\top \rr_0^* \right)^2 }{ \tilde{\mu} } \cdot \exp \left( - 2 \frac{\lambda_1 - \lambda_i}{\lambda_1} t \right).
  \end{align*}

  Given $\delta \in (0,1)$, define $S(\delta)=\{ i: \lambda_i > (1 - \delta) \lambda_1 \}$. For $\delta_1, \delta_2 \in (0,1)$, define
  \begin{align*}
    T(\delta_1,\delta_2) := \ceil{ \frac{1}{2 \delta_1} \log\left( \frac{1}{\tilde{\mu} \delta_2} \right) }.
  \end{align*}
  For all $i \not \in S(\delta_1)$, when $t > T(\delta_1,\delta_2)$, it holds that $(\e_i^\top \rr_t^*)^2 \le \delta_2 ( \e_i^\top \rr_0^*)^2$, and thus in particular $\sum_{i\in S(\alpha/2)} \left( \e_i^\top \rr_t^* \right)^2 \ge 1 - \delta_2$.

  Part one (crude regime) of the lemma now follows by noticing that, by setting $\delta_1=\delta_2=\frac{\alpha}{2}$ we have that for $t \ge T \left( \frac{\alpha}{2}, \frac{\alpha}{2}\right)$, it holds that
  \begin{align*}
    (\rr_t^*)^\top \M_{\lambda} \rr_t^* = \sum_{i=1}^d \lambda_i \left( \e_i^\top \rr_t^* \right)^2 
    \ge \sum_{i\in S(\alpha/2)} \left( 1-\frac{\alpha}{2} \right) \lambda_1 \left( \e_i^\top \rr_t^* \right)^2 \ge \left( 1-\frac{\alpha}{2} \right)^2 \lambda_1 \ge \left( 1- \alpha \right) \lambda_1 .
  \end{align*}

  For the second part (accurate regime) of the lemma, note that $S\left( \frac{\lambda_1 - \lambda_2}{\lambda_1} \right) = \{1\}$. Thus for all $t \ge T \left( \frac{\lambda_1 - \lambda_2}{\lambda_1}, \alpha \right)$, it holds that $( \e_1^\top \rr_t^* )^2 \ge 1 - \alpha$.

\end{proof}

\begin{proof}[\textbf{Proof of Lemma~\ref{lem:shift-and-invert-inexact-distance}}]
  We prove the bound for unnormalized iterates by induction. The case for $t=1$ holds trivially. For $t \ge 2$, we can bound the error of the unnormalized iterates using the exact solution to $\tilde{h}_t$: 
  \begin{align} \label{e:SI-split-errors-1}
    \norm{\tilde{\rr}_t - \tilde{\rr}_t^*} \le 
    \norm{\tilde{\rr}_t - \bar{\rr}_t} + 
    \norm{\bar{\rr}_t - \tilde{\rr}_t^*} .
  \end{align}

  The second term of~\eqref{e:SI-split-errors-1} is concerned with the error due to inexact target in the least squares problem $h_t(\uu,\vv)$ as $\left[ \begin{array}{c}  \uu_{t-1} \\ \vv_{t-1} \end{array} \right]$ is different from $\left[ \begin{array}{c}  \uu_{t-1}^* \\ \vv_{t-1}^* \end{array} \right]$. We can bound this term as 
  \begin{align} 
    \norm{\bar{\rr}_t - \tilde{\rr}_t^*}
    & = \norm{ \M_{\lambda} \rr_{t-1} - \M_{\lambda} \rr_{t-1}^* } 
    \le  \norm{ \M_{\lambda} } \cdot \norm{\rr_{t-1} - \rr_{t-1}^*} \nonumber \\ 
    \label{e:SI-split-errors-2}
    & = \sigma_1 (\M_{\lambda}) \cdot \norm{\rr_{t-1} - \rr_{t-1}^*}.
  \end{align}

  In view of the update rule of our algorithm and the triangle inequality, we have 
  \begin{align} 
    & \norm{ \rr_{t-1} -  \rr_{t-1}^* } \nonumber \\
    \le & \norm{\frac{\tilde{\rr}_{t-1}}{\norm{\tilde{\rr}_{t-1}}} - \frac{\tilde{\rr}_{t-1}}{\norm{\tilde{\rr}_{t-1}^*}} } + \norm{\frac{\tilde{\rr}_{t-1}}{\norm{\tilde{\rr}_{t-1}^*}} - \frac{\tilde{\rr}_{t-1}^*}{\norm{\tilde{\rr}_{t-1}^*}} } \nonumber \\
    = & \norm{\tilde{\rr}_{t-1}} \abs{ \frac{1}{\norm{\tilde{\rr}_{t-1}}} - \frac{1}{\norm{\tilde{\rr}_{t-1}^*}} }
    + \frac{1}{\norm{\tilde{\rr}_{t-1}^*}} \norm{\tilde{\rr}_{t-1} - \tilde{\rr}_{t-1}^*}  \nonumber \\
    = & \frac{1}{\norm{\tilde{\rr}_{t-1}^*}} \abs{ \norm{\tilde{\rr}_{t-1}^*} - \norm{\tilde{\rr}_{t-1}} } + \frac{1}{\norm{\tilde{\rr}_{t-1}^*}} \norm{\tilde{\rr}_{t-1} - \tilde{\rr}_{t-1}^*}  \nonumber \\ \label{e:SI-split-errors-3}
    \le & \frac{2}{\norm{\tilde{\rr}_{t-1}^*}} \norm{\tilde{\rr}_{t-1} - \tilde{\rr}_{t-1}^*}  
    \le \frac{2 \tilde{R}_{t-1}}{\norm{\tilde{\rr}_{t-1}^*}} .
  \end{align}
  For $t\ge 2$, we have $\tilde{\rr}_{t-1}^* =  \M_{\lambda} \rr_{t-2}^*$ and $\norm{ \rr_{t-2}^* }=1$, and thus
  \begin{align*}
    \norm{ \tilde{\rr}_{t-1}^*} \ge \sigma_d (\M_{\lambda}). % = \frac{1}{\lambda+\rho_1}.
  \end{align*}
  Combining~\eqref{e:SI-split-errors-1},~\eqref{e:SI-split-errors-2} and~\eqref{e:SI-split-errors-3} gives
  \begin{align*}
    \norm{\tilde{\rr}_t - \tilde{\rr}_t^*} & \le 
    \sqrt{\sigma_1 (\M_{\lambda}) \cdot \epsilon} + 2 \kappa_{\lambda} \tilde{R}_{t-1} = \tilde{R}_{t}.
    %% \frac{\sqrt{\te}}{\sqrt{\lambda - \rho_1}} + 2 \frac{\lambda+\rho_1}{\lambda-\rho_1} \tilde{R}_{t-1} = \tilde{R}_{t}.
  \end{align*}
  The bound for normalized iterates follows from~\eqref{e:SI-split-errors-3}.
\end{proof}

\begin{proof}[\textbf{Proof of Lemma~\ref{lem:shift-and-invert-inexact}}]

  For the first item (crude regime), observe that
  \begin{align} \label{e:SI-inexact-1}
    \rr_t^\top \M_{\lambda} \rr_t = (\rr_t^*)^\top \M_{\lambda} \rr_t^* + \left( (\rr_t^*)^\top \M_{\lambda} \rr_t^* - \rr_t^\top \M_{\lambda} \rr_t \right),
  \end{align}
  and that
  \begin{align*}
    \abs{ (\rr_t^*)^\top \M_{\lambda} (\rr_t^*) - \rr_t^\top \M_{\lambda} \rr_t }
    & = \abs{ \left( \M_{\lambda}^{\frac{1}{2}} \rr_t^* + \M_{\lambda}^{\frac{1}{2}} \rr_t \right)^\top \left( \M_{\lambda}^{\frac{1}{2}} \rr_t^* - \M_{\lambda}^{\frac{1}{2}} \rr_t \right) } \\
    & \le \norm{  \M_{\lambda}^{\frac{1}{2}} \rr_t^* + \M_{\lambda}^{\frac{1}{2}} \rr_t } \norm{  \M_{\lambda}^{\frac{1}{2}} \rr_t^* - \M_{\lambda}^{\frac{1}{2}} \rr_t } \\
    & \le \norm{  \M_{\lambda}^{\frac{1}{2}} } \norm{\rr_t^* + \rr_t } \norm{  \M_{\lambda}^{\frac{1}{2}} } \norm{\rr_t^* - \rr_t } \\
    & \le \norm{ \M_{\lambda}} (\norm{\rr_t^*} + \norm{ \rr_t }) \norm{\rr_t^* - \rr_t } \\
    & = 2 \sigma_1 (\M_{\lambda}) \cdot \norm{\rr_t^* - \rr_t }.
  \end{align*}
  Our choices of $T$ and $\te$ make sure that $(\rr_T^*)^\top \M_{\lambda} \rr_T^* \ge (1 - \frac{\alpha}{2}) \cdot \sigma_1 (\M_{\lambda})$ by Lemma~\ref{lem:shift-and-invert-exact} and that $\norm{\rr_T^* - \rr_T } \le R_T = \alpha/4$ by Lemma~\ref{lem:shift-and-invert-inexact-distance}. Continuing from ~\eqref{e:SI-inexact-1}, we have
  \begin{align*}
    \rr_T^\top \M_{\lambda} \rr_T \ge \left( 1 - \frac{\alpha}{2} \right) \cdot \sigma_1 (\M_{\lambda}) - \frac{\alpha}{2} \cdot \sigma_1 (\M_{\lambda}) = (1 - \alpha) \cdot \sigma_1 (\M_{\lambda}).
  \end{align*}

  For the second item (accurate regime), observe that
  \begin{align} \label{e:SI-inexact-2}
    (\rr_{t}^\top \rr^*)^2 = \left( (\rr_{t}^*)^\top \rr^* + (\rr_{t} - \rr_{t}^*)^\top \rr^*  \right)^2 \ge \left( (\rr_{t}^*)^\top \rr^* \right)^2 - 2 \norm{\rr_{t} - \rr_{t}^*}.
  \end{align}
  Our choices of $T$ and $\te$ make sure that $\left( (\rr_{T}^*)^\top \rr^* \right)^2 \ge 1 - \frac{\alpha}{2}$ by Lemma~\ref{lem:shift-and-invert-exact} and that $\norm{\rr_T^* - \rr_T } \le R_T = \alpha/4$ by Lemma~\ref{lem:shift-and-invert-inexact-distance}. Continuing from ~\eqref{e:SI-inexact-2}, we have
  \begin{align*}
    (\rr_T^\top \rr^*)^2 \ge 1 - \frac{\alpha}{2} - \frac{\alpha}{2} = 1 - \alpha.
  \end{align*}
\end{proof}

\subsection{Iteration complexity of Algorithm~\ref{alg:meta-shift-and-invert}}
\label{append:proof-shift-and-invert-iteration-complexity}

Observe that, the \textbf{for} loops within the \textbf{repeat-until} loop, as well as the final \textbf{for} loop in Algorithm~\ref{alg:meta-shift-and-invert}  are running inexact power iterations on $\M_{\lambda_{(s)}}$ and $\M_{\lambda_{(f)}}$ for $m_1$ and $m_2$ inexact matrix-vector multiplication respectively. And the convergence of inexact power iterations is provided by Lemma~\ref{lem:shift-and-invert-inexact-distance}.

For each iteration of the \textbf{repeat-until} loop, we work in the crude regime and only require $\rr_{s m_1}$ to give a constant multiple estimate of $\M_{\lambda_{(s)}}$. The lemma below shows an important property of $\Delta_s$ which is used to locate $\lambda_{(f)}$, and the number of iterations needed to reach $\lambda_{(f)}$.

\begin{lemma}[Iteration complexity of the \textbf{repeat-until} loop in Algorithm~\ref{alg:meta-shift-and-invert}] \label{lem:shift-and-invert-repeat-until}
  Suppose that $\tilde{\Delta} \in [c_1 \Delta,\ c_2 \Delta]$ where $c_2 \le 1$. Set $m_1 = \ceil{ 8 \log \left( \frac{16}{\mu^\prime} \right) }$ and $\te \le \frac{1}{3084} \left( \frac{\tilde{\Delta}}{18} \right)^{m_1-1}$ in Algorithm~\ref{alg:meta-shift-and-invert}. Then for all $s \ge 1$ it holds that
  \begin{align*}
    \frac{1}{2} (\lambda_{(s-1)} - \rho_1) \le \Delta_s \le \lambda_{(s-1)} - \rho_1,
  \end{align*}
  upon exiting this loop, the $\lambda_{(f)}$ satisfies
  \begin{align} \label{e:delta-s-interval}
    \rho_1 + \frac{\tilde{\Delta}}{4} \le \lambda_{(f)} \le \rho_1 + \frac{3 \tilde{\Delta}}{2},
  \end{align}
  and the number of iterations run by the \textbf{repeat-until} loop is $\log\left(\frac{1}{\tilde{\Delta}} \right)$.
  %% Moreover, all $\lambda_{(s)}$ used in Algorithm~\ref{alg:meta-shift-and-invert} satisfy that $\kappa_{\lambda_{(s)}} = \frac{\lambda_{(s)} + \rho_1}{\lambda_{(s)} - \rho_1} \le \frac{9}{\tilde{\Delta}}$.
\end{lemma}

\begin{proof}
  Let $\overline{\sigma}$ be an upper bound of all $\sigma_1 (\M_{\lambda_{(s)}})$ used in the \textbf{repeat-until} loop, \ie,
  \begin{align*}
    \overline{\sigma} \ge \sigma_1 (\M_{\lambda_{(s)}}), \qquad\quad s=1,2,\dots. 
  \end{align*}
  And suppose for now that throughout the loop, $\te$ satisfies
  \begin{align} \label{e:delta-s-1}
    \sqrt{ \overline{\sigma} \te } \le \frac{ \sigma_1 \left( \M_{\lambda_{(s-1)}} \right)}{8}.
  \end{align}

  Set $\alpha=\frac{1}{4}$ in Lemma~\ref{lem:shift-and-invert-inexact-distance} (crude regime), and with our choice of $m_1$ and 
  \begin{align} \label{e:epsilon-1}
    \te \le \frac{  \sigma_d (\M_{\lambda_{(s)}}) }{1024 \kappa_{\lambda_{(s)}}} \left( \frac{ 2 \kappa_{\lambda_{(s)}} - 1 }{ \left( 2 \kappa_{\lambda_{(s)}} \right)^{m_1} -1 } \right)^2,
  \end{align}
  we have
  \begin{align} \label{e:delta-s-2}
    \rr_{s m_1}^\top \M_{\lambda_{(s-1)}} \rr_{s m_1} \ge \frac{3}{4} \sigma_1 (\M_{\lambda_{(s-1)}}).
  \end{align}
  In view of the definition of the vector $\w_s$ in Algorithm~\ref{alg:meta-shift-and-invert}, and following the same argument in~\eqref{e:SI-bound-distance}, we have
  \begin{align*}
    \norm{ \frac{\z_s}{\sqrt{2}} - \M_{\lambda_{(s-1)}} \rr_{s m_1} } \le \sqrt{ \sigma_1 (\M_{\lambda_{(s-1)}}) \cdot \te }
  \end{align*}
  where $\z_s = \left[ \begin{array}{cc} \Sxx^{\frac{1}{2}} &  \\ & \Syy^{\frac{1}{2}} \end{array} \right] \w_s $.

  Then for every iteration of the \textbf{repeat-until} loop, it holds that
  \begin{align*}
    & \frac{1}{2} \left[\uu_{s m_1}^\top \vv_{s m_1}^\top\right] \left[ \begin{array}{cc} \Sxx & \\ & \Syy \end{array}\right] \w_s \\
    =\, & \rr_{s m_1}^\top \left( \frac{\z_s}{\sqrt{2}} \right) = \rr_{s m_1}^\top \M_{\lambda_{(s-1)}} \rr_{s m_1} + \rr_{s m_1}^\top \left( \frac{\z_s}{\sqrt{2}} - \M_{\lambda_{(s-1)}} \rr_{s m_1} \right)  \\
    \in\, & \left[ \rr_{s m_1}^\top \M_{\lambda_{(s-1)}} \rr_{s m_1} - \sqrt{ \sigma_1 (\M_{\lambda_{(s-1)}}) \cdot \te },\ \rr_{s m_1}^\top \M_{\lambda_{(s-1)}} \rr_{s m_1} + \sqrt{ \sigma_1 (\M_{\lambda_{(s-1)}}) \cdot \te } \right] \\
    \in\, & \left[ \rr_{s m_1}^\top \M_{\lambda_{(s-1)}} \rr_{s m_1} - \sqrt{ \overline{\sigma} \te },\ \rr_{s m_1}^\top \M_{\lambda_{(s-1)}} \rr_{s m_1} + \sqrt{ \overline{\sigma} \te } \right],
  \end{align*}
  where we have used the Cauchy-Schwarz inequality in the second step.

  In view of~\eqref{e:delta-s-1} and~\eqref{e:delta-s-2}, it follows that
  \begin{align*}
    & \frac{1}{2} \left[\uu_{s m_1}^\top \vv_{s m_1}^\top\right] \left[ \begin{array}{cc} \Sxx & \\ & \Syy \end{array}\right] \w_s - \sqrt{ \overline{\sigma} \te } \\
    \in & \left[ \rr_{s m_1}^\top \M_{\lambda_{(s-1)}} \rr_{s m_1} - 2 \sqrt{ \overline{\sigma} \te },\ \rr_{s m_1}^\top \M_{\lambda_{(s-1)}} \rr_{s m_1} \right] \\
    \in & \left[ \frac{1}{2} \sigma_1 (\M_{\lambda_{(s-1)}}),\ \sigma_1 (\M_{\lambda_{(s-1)}})\right].
  \end{align*}

  By the definition of $\Delta_s$ in Algorithm~\ref{alg:meta-shift-and-invert} and the fact that $\sigma_1 (\M_{\lambda_{(s-1)}}) = \frac{1}{\lambda_{(s-1)}-\rho_1}$, we have 
  \begin{align} \label{e:lambda-recurse}
    \Delta_s  = \frac{1}{2}\cdot \frac{1}{\frac{1}{2} \left[\uu_{s m_1}^\top \vv_{s m_1}^\top\right] \left[ \begin{array}{cc} \Sxx & \\ & \Syy \end{array}\right] \w_s - \sqrt{ \overline{\sigma} \te }} \in \left[ \frac{1}{2} \left( \lambda_{(s-1)}-\rho_1 \right),\ \lambda_{(s-1)}-\rho_1 \right].
  \end{align}

  And as a result,
  \begin{align*}
    \lambda_{(s)} = \lambda_{(s-1)} - \frac{\Delta_s}{2} 
    \ge \lambda_{(s-1)} - \frac{1}{2} \left( \lambda_{(s-1)} - \rho_1 \right)
    = \frac{\lambda_{(s-1)} + \rho_1}{2},
  \end{align*}
  and thus by induction (note $\lambda_{(0)} \ge \rho_1$) we have $\lambda_{(s)} \ge \rho_1$ throughout the \textbf{repeat-until} loop.

  From~\eqref{e:lambda-recurse} we also obtain
  \begin{align*}
    \lambda_{(s)} - \rho_1 =  \lambda_{(s-1)}  - \rho_1  - \frac{\Delta_s}{2} \le 
    \lambda_{(s-1)}  - \rho_1 - \frac{1}{4} \left( \lambda_{(s-1)} - \rho_1 \right)
    = \frac{3}{4} \left( \lambda_{(s-1)} - \rho_1 \right).
  \end{align*}

  To sum up, $\lambda_{(s)}$ approaches $\rho_1$ from above and the gap between $\lambda_{(s)}$ and $\rho_1$ reduces at the geometric rate of $\frac{3}{4}$. Thus after at most $t_3=\ceil{ \log_{3/4} \left( \frac{\tilde{\Delta}}{\lambda_{(0)} - \rho_1} \right) } \sim \calO\left( \log \left( \frac{1}{\tilde{\Delta}} \right) \right)$ iterations, we reach a $\lambda_{(t_3)}$ such that $\lambda_{(t_3)} - \rho_1 \le \tilde{\delta}$. And in view of~\eqref{e:lambda-recurse}, the \textbf{repeat-until} loop exits in the next iteration. Hence, the overall number of iterations is at most $t_3 + 1 = \calO \left( \frac{1}{\tilde{\Delta}} \right)$.

  We now analyze $\lambda_{(f)}$ and derive the interval it lies in. 
  Note that $\Delta_{f} \le \tilde{\Delta}$ and $\Delta_{f-1} > \tilde{\Delta}$ by the exiting condition. In view of~\eqref{e:lambda-recurse}, we have
  \begin{align*}
    \lambda_{(f)} - \rho_1 = \lambda_{(f-1)} - \rho_1 - \frac{\Delta_{f}}{2} \le 2 \Delta_f  - \frac{\Delta_{f}}{2} = \frac{3 \Delta_f}{2} \le \frac{3 \tilde{\Delta}}{2}.
  \end{align*}
  On the other hand,
  \begin{align} \label{e:lambda-recurse-2}
    \lambda_{(f)} - \rho_1 = \lambda_{(f-1)} - \rho_1 - \frac{\Delta_{f}}{2} \ge \lambda_{(f-1)} - \rho_1 -  \frac{1}{2} \left( \lambda_{(f-1)} - \rho_1 \right) = \frac{1}{2} \left( \lambda_{(f-1)} - \rho_1 \right).
  \end{align}
  If $f=1$, then by our choice of $\lambda_{(0)}$ we have that $\lambda_{(f)} - \rho_1 \ge \tilde{\Delta}$. Otherwise, by unfolding~\eqref{e:lambda-recurse-2} one more time, we have that
  \begin{align*} 
    \lambda_{(f)} - \rho_1 \ge \frac{1}{4} \left( \lambda_{(f-2)} - \rho_1 \right) 
    \ge \frac{\Delta_{f-1}}{4} \ge \frac{\tilde{\Delta}}{4}.
  \end{align*}
  Thus in both case, we have that $\lambda_{(f)} - \rho_1 \ge \frac{\tilde{\Delta}}{4}$ holds.

  It remains to give an explicit bound on $\te$ based on the two requirements~\eqref{e:delta-s-1} and~\eqref{e:epsilon-1}. 
  Since the $\lambda_{(s)}$ values are monotonically non-increasing and lower-bounded by $\rho_1 + \frac{\tilde{\Delta}}{4}$, we have 
  \begin{align*}
    \max_{s}\ \sigma_{1} (\M_{\lambda_{(s)}}) = \sigma_{1} (\M_{\lambda_{(f)}}) = \frac{1}{\lambda_{(f)} - \rho_1 } \le \frac{4}{\tilde{\Delta} } =: \overline{\sigma},
  \end{align*}
  and 
  \begin{align*}
    \min_{s}\ \sigma_{1} (\M_{\lambda_{(s)}}) & = \sigma_{1} (\M_{\lambda_{(0)}}) = \frac{1}{\lambda_{(0)} - \rho_1 } = \frac{1}{ 1 + \tilde{\Delta} - \rho_1 } \\
    & \ge \frac{1}{ 1 + c_2 \Delta - \Delta } \ge 1 + (1 - c_2) \Delta \ge 1 + \frac{1 - c_2}{c_2} \tilde{\Delta} := \underline{\sigma},
  \end{align*}
  where the first inequality holds since by definition of $\Delta$ it follows that $\rho_1=\rho_2+\Delta \ge \Delta$.

  Therefore, for the assumption~\eqref{e:delta-s-1} to hold, we just need 
  \begin{align}  \label{e:epsilon-2}
    \left( \frac{\underline{\sigma}}{8 \sqrt{\overline{\sigma}}} \right)^2 =  
    \frac{ \left( 1 + \frac{1 - c_2}{c_2} \tilde{\Delta} \right)^2  }{64 \cdot \frac{4}{\tilde{\Delta}}} 
    \ge \frac{ 1 }{64 \cdot \frac{4}{\tilde{\Delta}}} = \frac{\tilde{\Delta}}{256} \ge \te .
  \end{align}

  We now derive a lower bound of the right hand side of~\eqref{e:epsilon-1}. Notice
  \begin{align} \label{e:kappa-f}
    \kappa_{\lambda_{(s)}}  = \frac{\lambda_{(s)} + \rho_1}{\lambda_{(s)} - \rho_1} = 1 + \frac{2 \rho_1}{\lambda_{(s)} - \rho_1} \le   1 + 2 \rho_1 \overline{\sigma} \le 1 + 2 \overline{\sigma} \le \frac{9}{\tilde{\Delta}}.
  \end{align}

  On the other hand, 
  \begin{align*}
    \sigma_d (\M_{\lambda_{(s)}}) \ge \sigma_d (\M_{\lambda_{(0)}})  = \frac{1}{\lambda_{(0)} + \rho_1} = \frac{1}{1 + \tilde{\Delta} + \rho_1} \ge \frac{1}{3}.
  \end{align*}

  As a result, we have
  \begin{align}
    \frac{  \sigma_d (\M_{\lambda_{(s)}}) }{1024 \kappa_{\lambda_{(s)}}} \left( \frac{ 2 \kappa_{\lambda_{(s)}} - 1 }{ \left(2 \kappa_{\lambda_{(s)}} \right)^{m_1} -1 }\right)^2 &
    \ge \frac{1}{3084 \cdot \frac{9}{\tilde{\Delta}}} \left( \frac{ 2 \frac{9}{\tilde{\Delta}} - 1 }{ \left(2 \frac{9}{\tilde{\Delta}} \right)^{m_1} -1 }\right)^2 
    \ge \frac{ \left( \frac{17}{\tilde{\Delta}} \right)^2}{ 3084 \cdot \frac{9}{\tilde{\Delta}}  \cdot \left(\frac{18}{\tilde{\Delta}} \right)^{m_1} } \nonumber \\ \label{e:epsilon-3}
    & \ge \frac{1}{3084} \left( \frac{\tilde{\Delta}}{18} \right)^{m_1-1}.
  \end{align}
  Our final bound on $\te$ chooses the smaller of~\eqref{e:epsilon-2} and~\eqref{e:epsilon-3}.
\end{proof}

For the final \textbf{for} loop of Algorithm~\ref{alg:meta-shift-and-invert}, we work in the accurate regime of power iterations.

\begin{lemma}[Iteration complexity of the final \textbf{for} loop in Algorithm~\ref{alg:meta-shift-and-invert}]  \label{lem:shift-and-invert-for-loop}
  Suppose that $\tilde{\Delta} \in [c_1 \Delta, c_2 \Delta]$ where $c_2 \le 1$. Set $m_2 = \ceil{ \frac{5}{4}  \log \left( \frac{128}{\tilde{\mu} \eta^2} \right) }$ and $\te \le \frac{ \eta^4 }{4^{10}} \left( \frac{\tilde{\Delta}}{18} \right)^{m_2-1}$ in Algorithm~\ref{alg:meta-shift-and-invert}. 
Then the $(\uu_{T}, \vv_{T})$ output by Phase I satisfies
  \begin{align} \label{e:phaseI-2}
    \frac{1}{4} (\uu_{T}^\top \Sxx \uu^* + \vv_{T}^\top \Syy \vv^*)^2 \ge 1 - \frac{\eta^2}{64}. 
  \end{align}
\end{lemma}

\begin{proof}
  Notice when $\lambda= \rho_1 + c ( \rho_1 - \rho_2)$, we have
  \begin{align*}
    \delta(\M_{\lambda} ) = \frac{\sigma_1 (\M_{\lambda})}{ \sigma_1 (\M_{\lambda}) - \sigma_2 (\M_{\lambda}) } = 
    \frac{\frac{1}{\lambda - \rho_1}}{\frac{1}{\lambda - \rho_1}  - \frac{1}{\lambda-\rho_2} } = \frac{\lambda - \rho_2}{\rho_1 - \rho_2} = \frac{ \rho_1 + c (\rho_1 - \rho_2) - \rho_2}{\rho_1 - \rho_2} = c+1.
  \end{align*}
  In view of~\eqref{e:delta-s-interval}, $\lambda_{(f)} - \rho_1 \le \frac{3}{2} \tilde{\Delta} \le  \frac{3 c_2}{2} {\Delta} \le  \frac{3}{2} {\Delta}$, and thus $\delta(\M_{\lambda_{(f)}} ) \le \frac{5}{2}$. 

  Set $\alpha=\frac{\eta^2}{64}$ in Lemma~\ref{lem:shift-and-invert-inexact-distance} (accurate regime), and with our choice of $m_2$ and 
  \begin{align} \label{e:epsilon-4}
    \te \le \frac{  \eta^4 \cdot \sigma_d (\M_{\lambda_{(f)}}) }{64^3 \cdot \kappa_{\lambda_{(f)}}} \left( \frac{ 2 \kappa_{\lambda_{(f)}} - 1 }{ \left(2 \kappa_{\lambda_{(f)}} \right)^{m_2} -1 }\right)^2,
  \end{align}
  we are guaranteed to obtained the desired alignment.

  We now give a lower bound of the right hand side of \eqref{e:epsilon-4}. First,
  \begin{align*}
    \sigma_d (\M_{\lambda_{(f)}}) = \frac{1}{\lambda_{(f)} + \rho_1} \ge 
    \frac{1}{ \rho_1 + \frac{3}{2} \Delta + \rho_1} \ge \frac{1}{4}.
  \end{align*}
  Recall that we have proved in~\eqref{e:kappa-f} that $\kappa_{\lambda_{(f)}} \le \frac{9}{\tilde{\Delta}}$.
  Following a derivation similar to that of~\eqref{e:epsilon-3}, we have
  \begin{align}
    \frac{  \eta^4 \cdot \sigma_d (\M_{\lambda_{(f)}}) }{64^3 \cdot \kappa_{\lambda_{(f)}}} \left( \frac{ 2 \kappa_{\lambda_{(f)}} - 1 }{ \left(2 \kappa_{\lambda_{(f)}} \right)^{m_2} -1 }\right)^2 \ge
    \frac{ \eta^4 }{4^{10}} \left( \frac{\tilde{\Delta}}{18} \right)^{m_2-1},
  \end{align}
  and this explains the $\epsilon$ we set in the lemma.
\end{proof}

\begin{proof}[\textbf{Proof of Theorem~\ref{thm:shift-and-invert-phase-I}}]
  As shown in Lemma~\ref{lem:shift-and-invert-for-loop}, the \textbf{repeat-until} loop runs $\calO \left( \log \left( \frac{1}{\tilde{\Delta}} \right) \right) \sim \calO \left( \log \left( \frac{1}{\Delta} \right) \right)$ iterations, and inside each iteration, we run $m_1$ approximate matrix-vector multiplications. On the other hand, the final \textbf{for} loop runs $m_2$ approximate matrix-vector multiplications. By the definitions of $m_1$ and $m_2$, the total number of invocations of approximate matrix-vector multiplications/least squares problems is 
  \begin{align*}
    m_1  \cdot \log \left( \frac{1}{\Delta} \right) + m_2 
    \sim 
    \calO \left( \log\left(\frac{1}{\tilde{\mu}}\right) \log \left( \frac{1}{\Delta} \right) + 
      \log\left( \frac{1}{\tilde{\mu} \eta^2} \right) \right) 
    \sim \tilde{\calO} (1) .
  \end{align*}
\end{proof}

\section{Proof of Theorem~\ref{thm:shift-and-invert-phase-II}}
\label{append:proof-shift-and-invert-phase-II}

\begin{proof}
  Notice that the eigenvectors of $\M_{\lambda}$ form an orthonormal bases of $\bbR^{d_x + d_y}$. Thus when~\eqref{e:phaseI-2} holds, \ie,  the alignment between $\left[\begin{array}{c} \Sxx^{\frac{1}{2}} \tu_{T} \\ \Syy^{\frac{1}{2}} \tv_{T} \end{array}\right]$ and $\left[\begin{array}{c} \Sxx^{\frac{1}{2}} \uu^* \\ \Syy^{\frac{1}{2}} \vv^* \end{array}\right]$ is large, the alignments between $\left[\begin{array}{c} \Sxx^{\frac{1}{2}} \tu_{T} \\  \Syy^{\frac{1}{2}} \tv_{T} \end{array}\right]$ and other eigenvectors have to be small. In particular, the alignment bewteen $\left[\begin{array}{c} \Sxx^{\frac{1}{2}} \tu_{T} \\ \Syy^{\frac{1}{2}} \tv_{T} \end{array}\right]$ and the tailing eigenvector $\left[\begin{array}{c} \Sxx^{\frac{1}{2}} \uu^* \\ - \Syy^{\frac{1}{2}} \vv^* \end{array}\right]$ has to be small:
  \begin{align}\label{e:align-tailing}
    (\uu_{T}^\top \Sxx \uu^* - \vv_{T}^\top \Syy \vv^*)^2 \le \frac{\eta^2}{16}.
  \end{align}
  From~\eqref{e:align-tailing} and~\eqref{e:phaseI-2}, we have respectively
  \begin{gather*}
    - \frac{\eta}{4} \le \abs{\uu_{T}^\top \Sxx \uu^*}-\abs{\vv_{T}^\top \Syy \vv^*} \le \frac{\eta}{4}, \\
    \abs{\uu_{T}^\top \Sxx \uu^*} + \abs{\vv_{T}^\top \Syy \vv^*} \ge 2\sqrt{1- \frac{\eta^2}{64}} \ge 2 \left( 1  - \frac{\eta}{8} \right)
  \end{gather*}
  where we have used the fact that $\sqrt{1-x}\ge 1-\sqrt{x}$ for $x \in [0,1]$ in the second inequality.

  Averaging the above two inequalities gives
  \begin{align*}
    \abs{\uu_{T}^\top \Sxx \uu^*} \ge 1 - \frac{\eta}{4}, \qquad \abs{\vv_{T}^\top \Syy \vv^*} \ge 1 - \frac{\eta}{4}.
  \end{align*}
  Finally,
  \begin{align*}
    (\hu^\top \Sxx \uu^*)^2 + (\hv^\top \Syy \vv^*)^2  & =
    \frac{(\uu_{T}^\top \Sxx \uu^*)^2}{\uu_{T}^\top \Sxx \uu_{T}} + \frac{(\vv_{T}^\top \Syy \vv^*)^2}{\vv_{T}^\top \Syy \vv_{T}} \\
    & \ge (1 - \frac{\eta}{4})^2 \left(
      \frac{1}{\uu_{T}^\top \Sxx \uu_{T}} + \frac{1}{\vv_{T}^\top \Syy \vv_{T}}
    \right) \\
    & \ge \left( 1 - \frac{\eta}{4} \right)^2 \frac{4}{\uu_{T}^\top \Sxx \uu_{T}+\vv_{T}^\top \Syy \vv_{T} } \\
    & \ge 2 \left( 1 - \frac{\eta}{2} \right) = 2 - \eta
  \end{align*}
  where we have used the fact that $\frac{1}{x} + \frac{1}{y} \ge \frac{4}{x+y}$ in the first inequality, and \eqref{e:normal-3} in the second inequality. Then the theorem follows from the fact that $(\hu^\top \Sxx \uu^*)^2$ and $(\hv^\top \Syy \vv^*)^2$ can be at most $1$.
\end{proof}

\section{Condition number of  $h_t$ for SVRG}
\label{append:cond-h}

\begin{lemma} \label{lem:shift-and-invert-cond-svrg}
  Throughout Algorithm~\ref{alg:meta-shift-and-invert}, the condition number of $h_t$ for SVRG is at most $\frac{9/c}{\Delta} \tilde{\kappa}$, where 
  \begin{align*}
    \tilde{\kappa}:=   \frac{ \max\limits_i\, \max \left(\norm{\x_i}^2,\norm{\y_i}^2 \right)}{\min \left( \sigma_{\min}(\Sxx), \sigma_{\min}(\Syy) \right) }.
  \end{align*}
\end{lemma}

\begin{proof}
  The gradient Lipschitz constant of $h_t^i(\uu,\vv)$ is bounded by the largest eigenvalue (in absolute value) of its Hessian\footnote{We omit the regularization terms, which are typically very small, to have concise expressions.}
  \begin{align*}
    \Q_{\lambda}^i =
    \left[
      \begin{array}{cc}
        \lambda  \x_i \x_i^\top  & - \x_i \y_i^\top \\
        - \y_i \x_i^\top & \lambda \y_i \y_i^\top
      \end{array}
    \right],
  \end{align*}
  and the largest eigenvalue is defined as
  \begin{align*}
    \max_{\g_x \in \bbR^{d_x},\g_y \bbR^{d_y}} \ \beta := \abs{ [\g_x^\top,\g_y^\top] \Q_{\lambda}^i \left[ \begin{array}{c} \g_x \\ \g_y \end{array} \right] } 
    \qquad \text{s.t.} \quad \norm{\g_x}^2 + \norm{\g_y}^2 = 1.
  \end{align*}
  We have
  \begin{align*}
    \beta & = \abs{ \lambda (\g_x^\top \x_i)^2 + \lambda (\g_y^\top \y_i)^2 - 2 (\g_x^\top \x_i) (\g_y^\top \y_i) } \\
    & \le \lambda (\g_x^\top \x_i)^2 + \lambda (\g_y^\top \y_i)^2 + 2 \abs{\g_x^\top \x_i} \abs{\g_y^\top \y_i} \\
    & \le  \lambda (\g_x^\top \x_i)^2 + \lambda (\g_y^\top \y_i)^2 + (\g_x^\top \x_i)^2 + (\g_y^\top \y_i)^2 \\
    & = (\lambda + 1) \left( (\g_x^\top \x_i)^2 + (\g_y^\top \y_i) \right) \\
    & \le (\lambda + 1) \left( \norm{\g_x}^2 \norm{\x_i}^2 + \norm{\g_y}^2 \norm{\y_i}^2) \right) \\
    & \le (\lambda + 1) \max \left(\norm{\x_i}^2,  \norm{\y_i}^2 \right)
  \end{align*}
  where we have used the Cauchy-Schwarz inequality and the constraint in the third and the last inequality respectively. 

  It only remains to bound $\frac{\lambda + 1}{\lambda - \rho}$. Note that we have shown in Lemma~\ref{lem:shift-and-invert-repeat-until} that $\lambda \ge \rho_1 + \frac{\tilde{\Delta}}{4}$ throughout Algorithm~\ref{alg:meta-shift-and-invert}, and thus
  \begin{align*}
    \frac{\lambda + 1}{\lambda - \rho} = 1 + \frac{1 + \rho}{\lambda - \rho} \le 1 + \frac{2}{\lambda - \rho} \le 1 + 2 \frac{4}{\tilde{\Delta}} \le \frac{9}{\tilde{\Delta}} \le  \frac{9/c_1}{\Delta}.
  \end{align*}
\end{proof}

\section{More details of the experiments}
\label{append:expts}

The statistics of these datasets are summaized in Table~\ref{t:real_data}. These datasets have also been used by~\cite{Ma_15b,Wang_15c} for demonstrating their stochastic CCA algorithms. 

\begin{table}[h]
\centering
\caption{Brief summary of datasets.}
\label{t:real_data}
\begin{tabular}{|c|c|c|c|c|}
\hline
Datasets & Description & $d_x$ & $d_y$ & $N$ \\ \hline
Mediamill& Image and its labels & 100 & 120 & 30,000 \\
JW11     & Acoustic and articulation measurements & 273 & 112 & 30,000 \\
MNIST    & Left and right halves of images & 392 & 392 & 60,000 \\ \hline
\end{tabular}
\end{table}

We now provide additional details for the experiments. For \sAppGrad, both gradient and normalization steps are estimated with mini-batchs of $100$ samples (the authors of~\cite{Ma_15b} suggest that the mini-batch size shall be at least the same magnitude as the dimensionality of the CCA projection). For \ShiftVR\ and \ShiftAVR, within the \textbf{repeat-until} loop, we apply SVRG with $M=2$ epochs to approximately find the top eigenvector $\w_s$, and SVRG with $M=2$ epochs to approximately calculate its top eigenvalue of $\M_{\lambda_{(s)}}$ as $\w_s^T \M_{\lambda_{(s)}} \w_s$. We exit the \textbf{repeat-until} loop when $\Delta_s\le 0.06$. Afterwards, for the fixed $\lambda_{(f)}$, we apply SVRG to solve every least squares problems with $M=4$ epochs. Each epoch of SVRG includes a batch gradient evaluation and $m = N$ stochastic gradient steps. 
We set the step size according to the smoothness for each least squares solver, \ie, $\frac{1}{\sigma_{\max} (\Sxx)}$ for GD/AGD in \AppGrad/\sAppGrad/\CCALin, and $\frac{1}{\max_{i} \norm{ \x_i }^2}$ for SVRG/ASVRG in our algorithms.

\section{Other related work}
\label{sec:related}

Recent years have witnessed continuous efforts to scale up fundamental methods such as principal component analysis (PCA) and partial least squares with stochastic/online updates~\cite{WarmutKuzmin08a,Arora_12a,Balsub_13a,Shamir15a,Xie_15b,GarberHazan15c,Jin_15a}. 
% {Krasul69a,OjaKarhun85a,WarmutKuzmin08a,Arora_12a,Arora_13a,Mitliag_13a,Balsub_13a,Shamir15a,Xie_15b,GarberHazan15c,Jin_15a,Frostig_16a}. 
But as pointed out by~\cite{Arora_12a}, the CCA objective is more challenging due to the constraints. 
% Our algorithm is inspired by the stochastic PCA algorithm of~\cite{GarberHazan15c} which transforms the nonconvex PCA objective into a small number of well-conditioned regularized least squares problems (solved by SVRG) through shifting and inverting the covariance matrix and running power iterations on the transformed matrix. For CCA, the alternating least squares formulation (Algorithm~\ref{alg:als}) already reduces to solving convex regularized least squares problems.

\cite{Yger_12a} proposed an adaptive CCA algorithm with efficient online updates based on matrix manifolds defined by the constraints. However, the goal of their algorithm is anomaly detection for streaming data with a varying distribution, rather than to optimize the CCA objective on a given dataset. Similar to our algorithms, the stochastic CCA algorithms of~\cite{Ma_15b,Wang_15c} are motivated by the ALS formulation. \cite{Xie_15b} proposed a stochastic algorithm based on the Lagrangian formulation of the objective~\eqref{e:cca}. None of these online/stochastic algorithms have rigorous global convergence guarantee.

\end{document}